\PassOptionsToPackage{table}{xcolor}
\documentclass{article}
\usepackage{iclr2027_conference,times}
\setcitestyle{citesep={,},yysep={,}}

\usepackage{amsmath,amsfonts}
\usepackage{xcolor}
\usepackage{graphicx}
\usepackage{hyperref}
\usepackage{url}
\usepackage{xspace}
\usepackage{booktabs}
\usepackage{multirow}
\usepackage{wrapfig}
\usepackage{amsthm}
\usepackage{float}
\usepackage{placeins}

\newtheorem{lemma}{Lemma}[section]
\newtheorem{proposition}[lemma]{Proposition}
\newtheorem{corollary}[lemma]{Corollary}
\theoremstyle{remark}
\newtheorem{remark}[lemma]{Remark}

\newcommand{\method}{Lightning Weave\xspace}

\title{Lightning Weave: Improving the Accuracy--Efficiency Frontier of Reasoning Models through Capability Composition}

\author{Yecheng Wu\textsuperscript{1}, Song Han\textsuperscript{1,2}, Han Cai\textsuperscript{2}\\
\normalfont\textsuperscript{1}Massachusetts Institute of Technology \qquad \textsuperscript{2}NVIDIA\\
\texttt{wyc557@mit.edu}\\
\url{https://github.com/jet-ai-projects/Lightning-Weave}
}

\iclrfinalcopy
\DeclareMathSizes{7.5}{7}{5}{5}
\begin{document}

\maketitle
\lhead{Preprint}
\begin{abstract}
A core goal of efficient reasoning is to improve the accuracy--efficiency
frontier.
However, jointly improving reasoning accuracy and inference efficiency can be
challenging, as the two objectives can favor different reasoning behaviors.
Independently post-trained models already offer distinct strengths in
accuracy and efficiency. We introduce \method, a post-training
framework that extracts and composes these independently learned
capabilities in a single student through on-policy distillation. Each acquired
capability is represented by the policy shift from the model before
post-training to the resulting specialist. \method combines aligned log-ratio
shifts at shared student token states and uses Tilted-Target DOPD to convert the
cached signals into a stable learning target. Each anchor pair scores the
cached trajectories once, enabling subsequent student training without serving
multiple live anchor models concurrently. Across
diverse student models and benchmarks in mathematics and code, \method
substantially improves upon the base students and achieves a state-of-the-art
accuracy--efficiency frontier. On
Qwen3.5-4B, it raises HMMT 2025 accuracy from 59.2\% to 64.0\% with 10.7\%
fewer response tokens, and LiveCodeBench v5 accuracy from 41.7\% to 54.2\%
with 9.6\% fewer response tokens. Adjusting the relative strengths of the
anchor signals yields a strong empirical accuracy--efficiency
Pareto frontier. These results establish \method as a new practical route to
efficient reasoning through capability composition. Our code is released at \url{https://github.com/jet-ai-projects/Lightning-Weave}.
\end{abstract}

\begin{figure}[H]
    \centering
    \includegraphics[width=\textwidth]{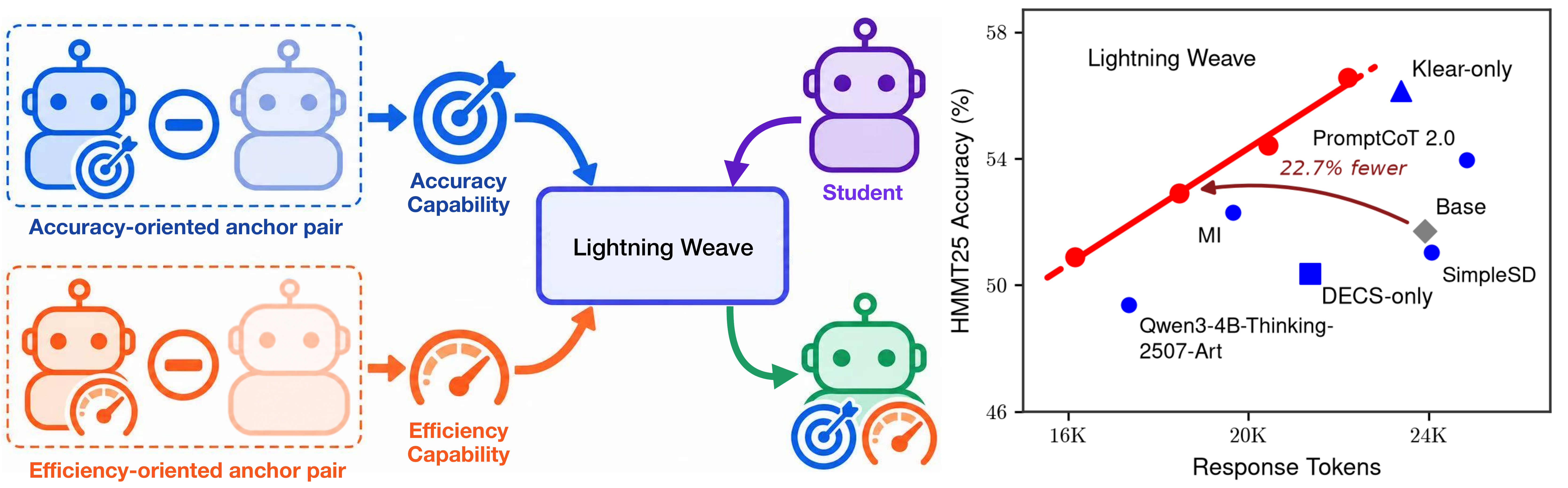}
    \caption{Left: Intuition of \method. Accuracy- and efficiency-oriented
    anchor pairs isolate complementary capabilities acquired through
    independent post-training. \method composes these capabilities into a
    joint target, transferring both to a single student and thereby improving
    its accuracy and efficiency simultaneously. Right: On
    Qwen3-4B-Thinking-2507, varying the relative strengths of the two
    capabilities yields a family of students that traces an empirical
    accuracy--efficiency Pareto frontier on HMMT 2025. The resulting frontier
    compares favorably with the base model, single-anchor policies, and
    existing efficient-reasoning methods.}
    \label{fig:teaser}
\end{figure}

\section{Introduction}
\label{sec:intro}

Large reasoning models have achieved strong performance on complex tasks,
but often generate lengthy reasoning traces that incur substantial inference
cost. Aggressively shortening these traces can compromise accuracy by
curtailing useful exploration. Efficient reasoning therefore aims to improve
the accuracy--efficiency frontier, achieving stronger performance for a given
inference budget. Existing approaches pursue this goal through length-aware
reinforcement learning and fine-tuning~\citep{jiang2026overthinking,liu2025dler,luo-etal-2026-o1},
training and distillation across reasoning modes~\citep{luo2026autol2s,liang2026orbit,ruan2026contrastive},
and merging deliberate and concise reasoning models in parameter
space~\citep{wu2025unlocking,yao2026activation,lan2025thinking}.

Directly optimizing accuracy and efficiency together can be challenging, as
the two objectives can favor different reasoning behaviors and improvements
in one may compromise the other. Meanwhile, independently post-trained models
already offer distinct strengths in reasoning accuracy~\citep{su2025klear}
and efficiency~\citep{jiang2026overthinking}. This suggests a complementary
route to efficient reasoning: extract what these specialists have already
learned and compose their capabilities in a single student through policy
distillation~\citep{rusu2015policy}.
The challenge is to reconcile their potentially competing effects on accuracy
and token use within the same reasoning task. This motivates our central question:
\emph{Can capabilities learned through independent post-training be extracted
and composed to improve a student's accuracy--efficiency frontier?}

We introduce \method, a capability-composition framework built on on-policy
distillation (OPD)~\citep{agarwal2024policy,gu2024minillm}. OPD supervises a
student on its own generated trajectories using a teacher's token-probability
distributions, and recent multi-teacher
approaches use this mechanism to integrate specialized
capabilities~\citep{ma2026mopd,yang2026nemotron,chen2026counteraction,gao2026open}.
To represent what each specialist has learned, we follow Direct On-Policy
Distillation (DOPD)~\citep{feng2026weak} and related policy-shift
formulations~\citep{yang2026learning,yu2026weak}. We pair each
specialist with its checkpoint before the relevant post-training stage,
forming an \emph{anchor pair}. The log-ratio of their token probabilities
represents the behavioral shift acquired during post-training. \method aligns
these shifts in the student's token space and combines them into one student
target, without requiring matching model parameters.

Composing multiple shifts online, however, makes supervision costly: with
$K$ anchor pairs, each new student rollout requires scoring by up to $2K$
anchor models. Lightning OPD~\citep{wu2026lightning} demonstrates an offline
pipeline that caches teacher feedback before student training, suggesting a
practical route to offline composition. Yet directly reusing DOPD's
sampled-token objective on cached trajectories does not generally preserve
the intended stationary target. As the student moves away from the behavior
policy that generated the cache, the regularizer estimated from cached
actions need not equal the intended KL penalty for the current student.
Consequently, even after the student reaches the intended shifted policy, the
cached surrogate can still produce a non-zero gradient and push the student
away from that solution.

We address this mismatch with Tilted-Target DOPD. Each cached shift tilts the
frozen behavior policy into an explicit target distribution, and training
minimizes the KL divergence from the current student to this target. For the
token-level DOPD surrogate, this construction preserves the expected initial
update and makes the loss and gradient vanish when the student matches the
target on cached states. The explicit target thus provides corrective feedback
as the student changes, giving offline capability transfer a well-defined
stationary solution.

Using this objective, \method composes the accuracy- and efficiency-oriented
shifts at every shared student token state. We first take a weighted sum of
the aligned log-ratios and then use it to construct one joint tilted target.
Composing shifts before target construction lets all capability sources
jointly supervise each student decision. Each anchor pair scores the same
cached trajectories once, after which optimization requires only the student
and cached scores. Adjusting the relative strengths of the anchor signals yields students with different accuracy--efficiency trade-offs.

We evaluate \method across diverse student models in separate mathematics
and code training settings, covering five benchmarks. Our primary setting
composes the accuracy-oriented shift from Klear~\citep{su2025klear} with the
efficiency-oriented shift from DECS~\citep{jiang2026overthinking}. Across the
fully evaluated students, composition improves the overall
accuracy--efficiency trade-off over the base and either single-anchor
alternative. On Qwen3-4B~\citep{yang2025qwen3}, \method raises AIME
2024~\citep{maa2026aime} accuracy from 72.9\% to 76.0\% while reducing
response tokens by 21.3\%. On Qwen3.5-4B~\citep{qwen3.5}, it raises HMMT
2025~\citep{hmmt2025february} accuracy from 59.2\% to 64.0\% with 10.7\%
fewer response tokens, and LiveCodeBench v5~\citep{jain2025livecodebench}
accuracy from 41.7\% to 54.2\% with 9.6\% fewer tokens.
Adjusting the relative strengths of the anchor signals during training yields
a family of students with a strong empirical accuracy--efficiency Pareto
frontier. These results support capability composition as a practical route
to more accurate and efficient reasoning.

\section{Related Work}
\label{sec:related}

\paragraph{Efficient Reasoning.}
Efficient reasoning improves the accuracy--efficiency frontier by controlling
inference computation. Existing approaches regulate token budgets and prune
low-confidence traces~\citep{han2025token,fu2025deepconf}, optimize length-aware
objectives~\citep{luo-etal-2026-o1,aggarwal2025l1,liu2025dler,jiang2026overthinking},
compress chain-of-thought traces~\citep{xia2025tokenskip}, or reason in
continuous representations~\citep{hao2025coconut,zhang2025softthinking}.
Others train or distill adaptive and budget-controlled reasoning
modes~\citep{luo2026autol2s,liang2026orbit,ruan2026contrastive}, or merge
deliberate and concise
checkpoints~\citep{wu2025unlocking,yao2026activation,lan2025thinking}.
Systems-oriented approaches include reward-guided and step-level
speculation~\citep{liao2025reward,pan2025specreason,fu2025lookahead},
certainty-guided allocation and scheduling in
Dynasor~\citep{fu2025certaindex}, and KV-cache compression in R-KV and
SkipKV~\citep{cai2025rkv,tian2026skipkv}.
These methods target generated tokens, latency, memory, and throughput.
\method targets the student's accuracy--efficiency frontier through
post-training capability composition, complementing inference-time systems
optimization.

\paragraph{Capability Composition.}
Capabilities can be integrated through policy distillation and knowledge
amalgamation~\citep{rusu2015policy,parisotto2015actor,teh2017distral,you2017learning,shen2019amalgamating,shen2019customizing},
heterogeneous output-space fusion~\citep{wan2024knowledge}, or parameter-space
model merging~\citep{li2022branch,wortsman2022model,matena2022merging,ilharco2022editing,yadav2023ties,yang2024adamerging,yu2024language}.
Multi-objective alignment studies reward-specialized interpolation and
controllable trade-offs~\citep{rame2023rewarded,yang2024rewards,zhong2024panacea,guo2024controllable},
while multi-teacher OPD integrates domain specialists and mitigates capability
interference~\citep{ma2026mopd,yang2026nemotron,chen2026counteraction,gao2026open}.
Rather than merging endpoint parameters or directly matching specialist
policies, \method composes aligned post-training policy shifts at shared
student token states. This decouples composition from parameter compatibility
while reconciling independently learned accuracy and efficiency capabilities
within the same reasoning task.

\paragraph{On-Policy Distillation.}
OPD provides dense teacher supervision on student-generated
trajectories~\citep{gu2024minillm,agarwal2024policy,lu2025onpolicydistillation,song2026survey}.
Recent work analyzes its failure
modes~\citep{li2026rethinking,fu2026revisiting,armandpour2026unmasking},
refines optimization~\citep{ko2026scaling,jin2026entropy,zhang2026reinforcement,wang2026not},
and enables cross-tokenizer transfer~\citep{he2026simpleopd,wang2026cross}.
DOPD and related formulations transfer policy shifts through teacher--reference
log-ratios or positive--negative model
contrasts~\citep{yang2026learning,feng2026weak,yu2026weak}.
Lightning OPD caches teacher feedback for offline training, while Lightning
OPD 2.0 corrects cross-teacher style
bias~\citep{wu2026lightning,wu2026lightning2}.
Building on policy-shift supervision and offline caching, Tilted-Target DOPD
provides an explicit, stable learning target on cached student states,
enabling capability composition without serving multiple live anchor models
during student training.

\section{Method}
\label{sec:method}

We develop \method to improve the accuracy--efficiency frontier by composing
independently learned capabilities in a single student.
As illustrated in Figure~\ref{fig:method-overview}, we represent each
capability as an anchor-pair policy shift, align these shifts at shared
student-generated token states, and compose them into a joint learning target.
We first review OPD and policy-shift supervision, then introduce
Tilted-Target DOPD to provide an explicit, stable target for offline
capability transfer, and finally build multi-anchor composition on this
foundation.

\begin{figure*}[t]
    \centering
    \includegraphics[width=\textwidth]{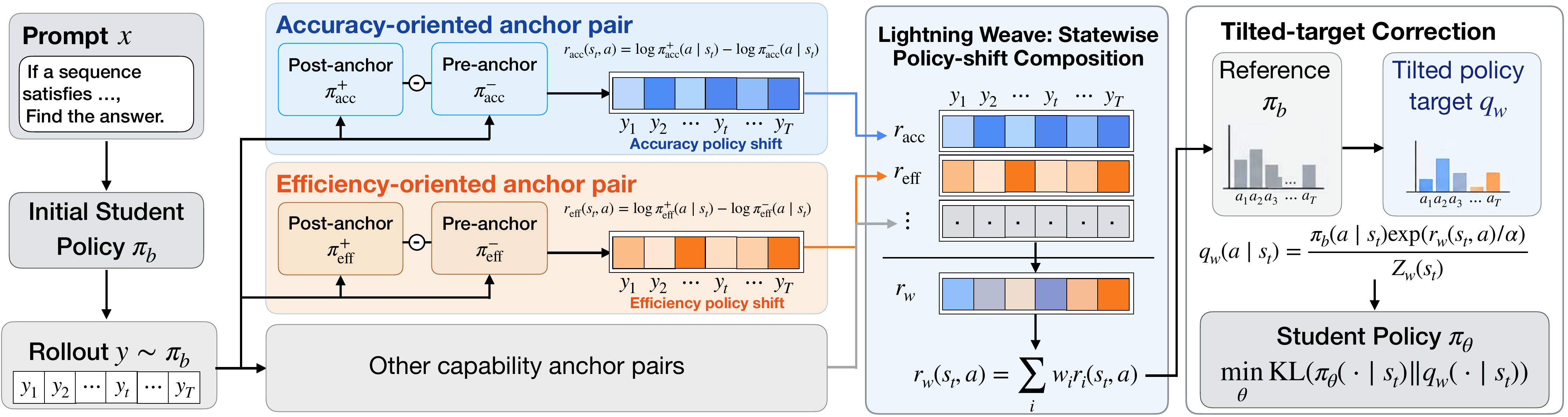}
    \caption{\textbf{Overview of \method.}
    (1) \textbf{Capability-specific policy shifts.} Given a prompt, the frozen
    behavior policy generates a rollout whose token states are shared by all
    independently trained anchor pairs. The log-ratio between each post-trained
    anchor and its pre-anchor isolates the corresponding capability-specific
    policy shift. (2) \textbf{Statewise multi-anchor composition.} \method
    combines the weighted shifts at every token state, allowing compatible
    changes to reinforce one another while reconciling competing changes
    locally. (3) \textbf{Tilted-target distillation.} The composed shift
    exponentially tilts the behavior policy into an explicit target
    distribution, which the student learns through KL minimization.}
    \label{fig:method-overview}
\end{figure*}

\subsection{Preliminaries}
\label{sec:preliminaries}

Let $x\sim\mathcal{D}$ be a prompt, $\hat{y}\sim\pi_\theta(\cdot\mid x)$ a student response, and $s_t=(x,\hat{y}_{<t})$ its token-level state. Given a teacher $\pi_T$, standard OPD minimizes the KL to the teacher on student-generated states~\citep{agarwal2024policy}:
\begin{equation}
    \mathcal{L}_{\mathrm{OPD}}(\theta)
    =
    \mathbb{E}_{x\sim\mathcal{D},\,\hat{y}\sim\pi_\theta}
    \left[
        \sum_t
        D_{\mathrm{KL}}
        \bigl(
            \pi_\theta(\cdot\mid s_t)
            \,\|\,\pi_T(\cdot\mid s_t)
        \bigr)
    \right].
    \label{eq:opd-objective}
\end{equation}

Direct On-Policy Distillation (DOPD)~\citep{feng2026weak} instead represents a specialist with a post-trained anchor and its pre-anchor, denoted by $(\pi^+,\pi^-)$. The sequence-level policy shift decomposes into dense token-level rewards:
\begin{equation}
    \Delta(\hat{y}\mid x)
    =
    \log\frac{\pi^+(\hat{y}\mid x)}{\pi^-(\hat{y}\mid x)}
    =
    \sum_t r_\Delta(\hat{y}_t\mid s_t),
    \qquad
    r_\Delta(a\mid s)
    =
    \log\frac{\pi^+(a\mid s)}{\pi^-(a\mid s)}.
    \label{eq:policy-shift}
\end{equation}
This ratio isolates the behavioral change introduced by post-training rather than the endpoint policy $\pi^+$. Let $\pi_b$ denote the initial student. In its idealized sequence-level form, DOPD maximizes
\begin{equation}
    \mathcal{J}_{\mathrm{DOPD}}(\theta)
    =
    \mathbb{E}_{x\sim\mathcal{D}}
    \left[
        \mathbb{E}_{\hat{y}\sim\pi_\theta(\cdot\mid x)}
        [\Delta(\hat{y}\mid x)]
        -
        \alpha D_{\mathrm{KL}}
        \bigl(
            \pi_\theta(\cdot\mid x)
            \,\|\,\pi_b(\cdot\mid x)
        \bigr)
    \right],
    \label{eq:dopd-objective}
\end{equation}
where $\alpha>0$ controls the KL penalty. Following DOPD's zero-discount token-level surrogate, we use $r_\Delta(a\mid s)$ as the immediate reward at each visited state.

\subsection{Single-Anchor Distillation: Tilted-Target DOPD}
\label{sec:single-anchor}

\paragraph{Offline DOPD and its limitation.}
Online DOPD must serve both anchors to evaluate $r_\Delta$ on the states visited by the current student. Inspired by offline score caching~\citep{wu2026lightning}, we instead sample trajectories once from $\pi_b$ and cache their anchor shifts. Let $d_{\pi_b}$ denote the resulting token-state distribution. A naive cached surrogate for Eq.~\ref{eq:dopd-objective} is
\begin{equation}
    \mathcal{J}_{\mathrm{naive}}(\theta)
    =
    \mathbb{E}_{s\sim d_{\pi_b}}
    \left[
        \mathbb{E}_{a\sim\pi_\theta(\cdot\mid s)}
        [r_\Delta(a\mid s)]
        -
        \alpha\widehat{D}_{\mathrm{KL}}^{\,b}
        \bigl(
            \pi_\theta(\cdot\mid s)
            \,\|\,\pi_b(\cdot\mid s)
        \bigr)
    \right],
    \label{eq:naive-offline-dopd}
\end{equation}
where $\widehat{D}_{\mathrm{KL}}^{\,b}$ denotes DOPD's low-variance sampled-token KL estimator evaluated on a cached action $a_b\sim\pi_b(\cdot\mid s)$. In particular, $\widehat{D}_{\mathrm{KL}}^{\,b}=\exp(\ell_\theta)-1-\ell_\theta$, where $\ell_\theta=\log\pi_b(a_b\mid s)-\log\pi_\theta(a_b\mid s)$. At initialization, the cached states and actions are on-policy, so Eq.~\ref{eq:naive-offline-dopd} matches the online DOPD update. As $\pi_\theta$ departs from $\pi_b$, however, cached actions no longer yield an on-policy estimate of the KL regularizer. The regularization is therefore not guaranteed to cancel the fixed policy-shift update at the intended solution, allowing training to continue pushing after the shift has been absorbed. The cached states remain an offline approximation. Our correction instead restores a well-defined fixed point on these states.

\paragraph{Tilted-Target DOPD.}
We make the destination of the regularized update explicit. Using the full-distribution form, each cached state $s\sim d_{\pi_b}$ defines
\begin{align}
    q(\cdot\mid s)
    &=
    \arg\max_{\pi}
    \left\{
        \mathbb{E}_{a\sim\pi(\cdot\mid s)}
        [r_\Delta(a\mid s)]
        -
        \alpha D_{\mathrm{KL}}
        \bigl(
            \pi(\cdot\mid s)
            \,\|\,\pi_b(\cdot\mid s)
        \bigr)
    \right\},
    \nonumber\\
    q(a\mid s)
    &=
    \frac{
        \pi_b(a\mid s)
        \exp\!\left(r_\Delta(a\mid s)/\alpha\right)
    }{Z(s)}.
    \label{eq:tilted-target}
\end{align}
Here $Z(s)=\sum_a\pi_b(a\mid s)\exp(r_\Delta(a\mid s)/\alpha)$ normalizes the distribution over actions. A larger $\alpha$ keeps $q$ closer to $\pi_b$, while a smaller $\alpha$ applies a stronger anchor shift. We then use the same student-to-target KL form as OPD, evaluated on the cached states:
\begin{equation}
    \mathcal{L}_{\mathrm{TT}}(\theta)
    =
    \alpha\,
    \mathbb{E}_{s\sim d_{\pi_b}}
    \left[
        D_{\mathrm{KL}}
        \bigl(
            \pi_\theta(\cdot\mid s)
            \,\|\,q(\cdot\mid s)
        \bigr)
    \right].
    \label{eq:tilted-target-loss}
\end{equation}
Substituting Eq.~\ref{eq:tilted-target} gives
\begin{equation}
    \alpha D_{\mathrm{KL}}(\pi_\theta\,\|\,q)
    =
    \alpha D_{\mathrm{KL}}(\pi_\theta\,\|\,\pi_b)
    -
    \mathbb{E}_{a\sim\pi_\theta}
    [r_\Delta(a\mid s)]
    +
    \alpha\log Z(s).
    \label{eq:tilted-target-equivalence}
\end{equation}
Because $\log Z(s)$ is independent of $\theta$,
Eq.~\ref{eq:tilted-target-loss} is equivalent to the full-distribution,
KL-regularized DOPD objective on each cached state. It preserves the expected
per-state online update at initialization and, unlike the naive surrogate, has
an explicit fixed point: the loss and gradient vanish when $\pi_\theta=q$.
Repeated optimization therefore approaches a well-defined policy target rather
than continually extrapolating the cached shift. This self-correcting behavior
stabilizes offline training.

\subsection{Multi-Anchor Distillation: \method}
\label{sec:multi-anchor}

\paragraph{Aligned policy-shift composition.}
Having established stable transfer from one anchor pair, we now seek a single
student target that realizes the changes acquired by $K$ independently trained
specialists. The $i$-th pair $(\pi_i^+,\pi_i^-)$ induces the token-level shift
$r_i(a\mid s)=\log\frac{\pi_i^+(a\mid s)}{\pi_i^-(a\mid s)}$.
We evaluate every pair on the same states cached from $\pi_b$ and express its scores over a common student action space. At each student decision, the resulting $\{r_i(a\mid s)\}_{i=1}^{K}$ therefore describe how different post-training procedures reweight the same action distribution. Because these changes are expressed as log-density ratios, their weighted composition is
\begin{equation}
    r_{\mathbf{w}}(a\mid s)
    =
    \sum_{i=1}^{K}w_i r_i(a\mid s)
    =
    \log
    \prod_{i=1}^{K}
    \left(
        \frac{\pi_i^+(a\mid s)}{\pi_i^-(a\mid s)}
    \right)^{w_i},
    \label{eq:composed-shift}
\end{equation}
where $\mathbf{w}=(w_1,\ldots,w_K)$ contains non-negative capability weights
that satisfy $\sum_{i=1}^{K}w_i=1$. Their values determine the balance among
capabilities, and we use $w_i=1/K$ for equal composition. With the same
$\alpha$ and optimization configuration, setting one weight to one and all
others to zero exactly recovers the corresponding single-anchor objective.

\paragraph{\method.}
We apply the composed shift once to the common behavior policy, defining
\begin{equation}
    q_{\mathbf{w}}(\cdot\mid s)
    =
    \arg\max_{\pi}
    \left\{
        \mathbb{E}_{a\sim\pi(\cdot\mid s)}
        [r_{\mathbf{w}}(a\mid s)]
        -
        \alpha D_{\mathrm{KL}}
        \bigl(
            \pi(\cdot\mid s)
            \,\|\,\pi_b(\cdot\mid s)
        \bigr)
    \right\}.
    \label{eq:maps-objective}
\end{equation}
The objective is strictly concave in $\pi$ and has the unique solution
\begin{equation}
    q_{\mathbf{w}}(a\mid s)
    =
    \frac{\pi_b(a\mid s)}{Z_{\mathbf{w}}(s)}
    \exp\!\left(
        \frac{1}{\alpha}
        \sum_{i=1}^{K}w_i r_i(a\mid s)
    \right)
    =
    \frac{\pi_b(a\mid s)}{Z_{\mathbf{w}}(s)}
    \prod_{i=1}^{K}
    \left(
        \frac{\pi_i^+(a\mid s)}{\pi_i^-(a\mid s)}
    \right)^{w_i/\alpha}.
    \label{eq:maps-target}
\end{equation}
Here $Z_{\mathbf{w}}(s)=\sum_a\pi_b(a\mid s)\exp(r_{\mathbf{w}}(a\mid s)/\alpha)$ normalizes $q_{\mathbf{w}}$ over the common student action space. Equation~\ref{eq:maps-target} is a product of relative policy changes around one shared student prior, not an arithmetic mixture of teacher distributions. Anchor pairs that support the same action reinforce one another, while conflicting changes are reconciled locally through their weighted ratios. In the sample-mixing baseline, each example is supervised by only one anchor-specific target, leaving their interaction implicit in shared model parameters. Our construction specifies that interaction directly at every aligned token state.
We then fit the student to this joint target using
\begin{equation}
    \mathcal{L}_{\mathrm{joint}}(\theta,\mathbf{w})
    =
    \alpha\,
    \mathbb{E}_{s\sim d_{\pi_b}}
    \left[
        D_{\mathrm{KL}}
        \bigl(
            \pi_\theta(\cdot\mid s)
            \,\|\,q_{\mathbf{w}}(\cdot\mid s)
        \bigr)
    \right].
    \label{eq:maps-loss}
\end{equation}
Expanding this KL shows that, up to a $\theta$-independent constant, the
objective combines the weighted DOPD directions under a single KL constraint,
while its loss and gradient vanish at $q_{\mathbf{w}}$.
Equation~\ref{eq:maps-loss} recovers
Tilted-Target DOPD when $K=1$. Each anchor pair is scored only once. After the
aligned shifts are cached, training serves only the student regardless of the
number of anchors. Complete derivations are provided in
Appendix~\ref{app:theory}, and practical finite-target approximation and
cross-tokenizer alignment are described in
Appendix~\ref{app:cross-tokenizer}.

\section{Experiments}
\label{sec:experiments}

\begin{table}[H]
\centering
\caption{Main accuracy--efficiency results across mathematics and code. We report accuracy (Acc., \%) and mean response tokens (\#Tok.) for each benchmark, together with average accuracy, response tokens, and AES over all five benchmarks. All \method rows compose the Klear and DECS policy shifts. Within each student block, the best value for each metric is marked in \textbf{bold} and the second-best is \underline{underlined}.}
\label{tab:main-results}
\resizebox{\textwidth}{!}{%
\begin{tabular}{lccccccccccccc}
\toprule
\multirow{2}[4]{*}{\textbf{Policy}} &
    \multicolumn{2}{c}{\textbf{AIME 2024}} &
    \multicolumn{2}{c}{\textbf{AIME 2025}} &
    \multicolumn{2}{c}{\textbf{HMMT 2025}} &
    \multicolumn{2}{c}{\textbf{LCB v5}} &
    \multicolumn{2}{c}{\textbf{LCB v6}} &
    \multicolumn{3}{c}{\textbf{Average}} \\
\cmidrule(r){2-3}\cmidrule(r){4-5}\cmidrule(r){6-7}\cmidrule(r){8-9}\cmidrule(r){10-11}\cmidrule(l){12-14}
& \textbf{Acc.}$\uparrow$ & \textbf{\#Tok.}$\downarrow$
& \textbf{Acc.}$\uparrow$ & \textbf{\#Tok.}$\downarrow$
& \textbf{Acc.}$\uparrow$ & \textbf{\#Tok.}$\downarrow$
& \textbf{Acc.}$\uparrow$ & \textbf{\#Tok.}$\downarrow$
& \textbf{Acc.}$\uparrow$ & \textbf{\#Tok.}$\downarrow$
& \textbf{Acc.}$\uparrow$ & \textbf{\#Tok.}$\downarrow$ & \textbf{AES}$\uparrow$ \\
\midrule
\multicolumn{14}{l}{\textbf{\textit{Qwen3-1.7B}}} \\
Base
    & 45.8 & 17,564 & 40.4 & 17,537 & 21.7 & 17,570
    & 35.1 & 14,796 & 30.3 & 14,409
    & 34.66 & 16,375 & 0.00 \\
Klear-only
    & \underline{53.2} & 18,457 & \textbf{41.7} & 19,178 & \textbf{24.1} & 19,751
    & \textbf{38.3} & 15,131 & \underline{31.9} & 15,198
    & \underline{37.84} & 17,543 & \underline{0.20} \\
DECS-only
    & 46.2 & \textbf{14,558} & 34.9 & \textbf{14,534} & 20.9 & \textbf{14,160}
    & 33.1 & \textbf{13,406} & 30.5 & \textbf{13,039}
    & 33.12 & \textbf{13,939} & $-0.08$ \\
\rowcolor[rgb]{.867,.922,.969}
\method
    & \textbf{55.1} & \underline{15,289} & \underline{40.8} & \underline{16,151} & \underline{22.9} & \underline{15,855}
    & \underline{38.1} & \underline{14,284} & \textbf{32.8} & \underline{13,680}
    & \textbf{37.94} & \underline{15,052} & \textbf{0.34} \\
\midrule
\multicolumn{14}{l}{\textbf{\textit{Qwen3-4B}}} \\
Base
    & 72.9 & 14,641 & 65.4 & 17,844 & 41.4 & 17,626
    & 53.0 & 13,843 & 46.8 & 14,663
    & 55.90 & 15,723 & 0.00 \\
Klear-only
    & \underline{75.8} & 15,061 & \textbf{71.3} & 17,715 & \textbf{45.6} & 18,230
    & \underline{57.4} & 13,252 & \underline{49.6} & 13,776
    & \textbf{59.94} & 15,607 & \underline{0.23} \\
DECS-only
    & 68.3 & \textbf{10,918} & 60.6 & \textbf{12,429} & 35.0 & \textbf{11,862}
    & 53.0 & \underline{12,531} & 46.2 & \underline{13,437}
    & 52.62 & \textbf{12,235} & $-0.09$ \\
\rowcolor[rgb]{.867,.922,.969}
\method
    & \textbf{76.0} & \underline{11,521} & \underline{70.1} & \underline{13,816} & \underline{43.0} & \underline{13,743}
    & \textbf{59.2} & \textbf{11,702} & \textbf{51.3} & \textbf{12,357}
    & \underline{59.92} & \underline{12,628} & \textbf{0.41} \\
\midrule
\multicolumn{14}{l}{\textbf{\textit{Qwen3-4B-Thinking-2507}}} \\
Base
    & 77.9 & 19,442 & 73.9 & 21,170 & 51.7 & 23,906
    & 65.5 & \underline{16,693} & 53.4 & \underline{18,105}
    & 64.47 & 19,863 & 0.00 \\
Klear-only
    & \underline{83.0} & 18,128 & \underline{78.5} & 19,792 & \textbf{56.1} & 23,377
    & \underline{66.4} & 17,853 & \textbf{56.7} & 19,487
    & \underline{68.16} & 19,727 & \underline{0.18} \\
DECS-only
    & 79.0 & \underline{17,680} & 75.5 & \underline{19,460} & 50.4 & \underline{21,369}
    & 63.9 & \textbf{16,004} & 53.2 & \textbf{17,390}
    & 64.40 & \underline{18,381} & 0.04 \\
\rowcolor[rgb]{.867,.922,.969}
\method
    & \textbf{85.3} & \textbf{16,003} & \textbf{80.4} & \textbf{17,970} & \underline{54.4} & \textbf{20,440}
    & \textbf{68.1} & 17,216 & \underline{56.5} & 18,790
    & \textbf{68.94} & \textbf{18,084} & \textbf{0.28} \\
\midrule
\multicolumn{14}{l}{\textbf{\textit{Qwen3.5-4B}}} \\
Base
    & 72.9 & \underline{18,316} & 63.5 & \underline{20,322} & \underline{59.2} & 26,481
    & 41.7 & 31,287 & 38.5 & 30,789
    & 55.16 & 25,439 & 0.00 \\
Klear-only
    & \underline{78.2} & 21,522 & \underline{66.7} & 23,346 & 56.7 & 27,096
    & 51.2 & 29,891 & \underline{46.8} & 29,638
    & \underline{59.92} & 26,299 & \underline{0.25} \\
DECS-only
    & 64.2 & \textbf{15,052} & 55.8 & \textbf{17,335} & 52.1 & \textbf{22,228}
    & \underline{52.4} & \textbf{26,549} & 45.8 & \textbf{27,091}
    & 54.06 & \textbf{21,651} & 0.06 \\
\rowcolor[rgb]{.867,.922,.969}
\method
    & \textbf{84.2} & 18,900 & \textbf{74.6} & 20,758 & \textbf{64.0} & \underline{23,635}
    & \textbf{54.2} & \underline{28,294} & \textbf{47.3} & \underline{28,841}
    & \textbf{64.86} & \underline{24,086} & \textbf{0.61} \\
\midrule
\multicolumn{14}{l}{\textbf{\textit{OLMo-3-7B-Think}}} \\
Base
    & 73.5 & 17,521 & 65.3 & 19,247 & 44.4 & 20,960
    & 52.1 & \underline{22,377} & 46.6 & \underline{22,679}
    & 56.38 & 20,557 & 0.00 \\
Klear-only
    & \underline{77.6} & 16,787 & \textbf{72.3} & 18,553 & \textbf{50.3} & 23,976
    & \underline{55.0} & 23,794 & 47.5 & 24,677
    & \textbf{60.54} & 21,557 & \underline{0.18} \\
DECS-only
    & 70.1 & \underline{13,497} & 61.9 & \textbf{14,528} & 39.2 & \textbf{16,076}
    & 52.2 & \textbf{20,857} & \underline{47.7} & \textbf{21,673}
    & 54.22 & \textbf{17,326} & $-0.04$ \\
\rowcolor[rgb]{.867,.922,.969}
\method
    & \textbf{77.9} & \textbf{13,196} & \underline{69.2} & \underline{14,545} & \underline{46.1} & \underline{16,487}
    & \textbf{58.0} & 22,404 & \textbf{49.4} & 22,873
    & \underline{60.12} & \underline{17,901} & \textbf{0.34} \\
\bottomrule
\end{tabular}%
}
\end{table}

\subsection{Experimental Setup}
\label{sec:experimental-setup}

\paragraph{Models and anchor pairs.}
We conduct experiments on five student models: Qwen3-1.7B, Qwen3-4B, Qwen3-4B-Thinking-2507~\citep{yang2025qwen3}, Qwen3.5-4B~\citep{qwen3.5}, and OLMo-3-7B-Think~\citep{olmo2025olmo}, spanning multiple model sizes and families. We represent each capability using the shift from a pre-anchor to its post-trained counterpart. Our accuracy-oriented anchor pair is Qwen3-8B-Base $\rightarrow$ Klear-Reasoner-8B (Klear)~\citep{su2025klear}, and our efficiency-oriented anchor pair is DeepSeek-R1-Distill-Qwen-1.5B $\rightarrow$ DECS-1.5B (DECS)~\citep{jiang2026overthinking}. We evaluate both single-anchor shifts and their Klear--DECS composition for every student. Section~\ref{sec:anchor-generalization} further extends this study to MiMo-RL as an additional anchor.

\paragraph{Training data.}
For mathematics, we sample 3,200 prompts from the math split of Skywork-OR1-RL-Data~\citep{he2025skywork}, converted to the DAPO prompt format~\citep{yu2026dapo}. For code, we sample 3,200 prompts from KlearReasoner-CodeSub-15K~\citep{su2025klear} after length filtering and prompt deduplication. Each student independently generates four responses per prompt, with a maximum response length of 2,048 tokens. All anchor pairs score the same student trajectories. When an anchor and student use different tokenizers, we project the anchor scores into the student token space through cross-tokenizer alignment. Implementation details are provided in Appendix~\ref{app:cross-tokenizer}.

\paragraph{Benchmarks and metrics.}
We evaluate mathematical reasoning on AIME 2024, AIME 2025~\citep{maa2026aime}, and HMMT February 2025~\citep{hmmt2025february}, and code generation on LiveCodeBench (LCB) v5 and v6~\citep{jain2025livecodebench}. For each math problem, we sample 64 responses at temperature $0.6$ and top-$p=0.95$, with a maximum of 32,768 generated tokens, and report mean answer accuracy over all samples. For each LiveCodeBench problem, we similarly sample four responses at temperature $0.6$ and top-$p=0.95$, with a 40,960-token limit, and report their mean accuracy (Avg@4). Alongside accuracy, we report the mean number of generated response tokens as a measure of inference efficiency. We also report the Accuracy--Efficiency Score (AES)~\citep{luo-etal-2026-o1}, computed against the corresponding base model per benchmark and then averaged equally across all five benchmarks.

\paragraph{Training settings.}
We optimize over the cached trajectories using Adam with a global batch size
of 64 and a constant learning rate of $1.0{\times}10^{-6}$. We use
$\alpha=2.0$ for both single- and multi-anchor training. Multi-anchor weights
are normalized to sum to one. Equal two-anchor composition assigns a weight of
$0.5$ to each anchor. Additional hyperparameters are provided in
Appendix~\ref{app:training-hyperparameters}.

\subsection{Main Results}
\label{sec:main-results}

Table~\ref{tab:main-results} reports evaluation results for five student models across five math and code benchmarks. Compared with both the base models and single-anchor policies, \method consistently improves the balance between accuracy and efficiency across model sizes and families. On Qwen3-4B, it improves average accuracy over the base model by 4.02 points while reducing response tokens by 19.7\%. On Qwen3-4B-Thinking-2507, \method improves average accuracy by 4.5 points while reducing response tokens by 9.0\%, outperforming both single-anchor policies in both aggregate metrics. On Qwen3.5-4B, \method yields a substantially larger accuracy gain of 9.70 points while still reducing response tokens by 5.3\%. In particular, its HMMT 2025 accuracy increases from 59.2\% to 64.0\%, reaching a level competitive with state-of-the-art models of a similar scale while using 10.7\% fewer tokens. Figure~\ref{fig:accuracy-efficiency-tradeoff}(a) further shows that \method achieves the highest AES and accuracy on all five benchmarks for Qwen3.5-4B. The gains also transfer to OLMo-3-7B-Think, where \method improves average accuracy by 3.74 points and reduces response tokens by 12.9\%. Across all five backbones, \method achieves the highest AES, substantially outperforming both single-anchor policies and delivering a consistently stronger accuracy--efficiency trade-off. Together, these results establish capability composition as a consistent route to improving the accuracy--efficiency frontier across model scales and families.

\subsection{Accuracy--Efficiency Frontiers}
\label{sec:accuracy-efficiency-frontiers}

To assess controllability, we sweep the relative Klear--DECS weights on
Qwen3-4B-Thinking-2507 under a fixed training budget.
Table~\ref{tab:q4-thinking-frontier} reports five interior compositions and
representative baselines: SimpleSD~\citep{zhang2026embarrassinglysimpleselfdistillationimproves},
Qwen3-4B-Thinking-2507-Art~\citep{wu2026artefficientreasoningdata},
PromptCoT 2.0~\citep{zhao2025promptcot20scalingprompt}, and model interpolation
(MI)~\citep{wu2026revisitingmodelinterpolationefficient}.

\begin{wraptable}{R}{0.66\textwidth}
\vspace{-2.0\baselineskip}
\centering
\caption{Accuracy--efficiency frontiers on
Qwen3-4B-Thinking-2507. We report the Klear--DECS sweep and representative
baselines on three mathematical reasoning benchmarks. The non-dominated
\method operating points trace an empirical Pareto frontier across all three
benchmarks.}
\vspace{0.5\baselineskip}
\label{tab:q4-thinking-frontier}
\resizebox{\linewidth}{!}{%
\begin{tabular}{lcccccc}
\toprule
\multirow{2}[4]{*}{\textbf{Policy}} &
    \multicolumn{2}{c}{\textbf{AIME 2024}} &
    \multicolumn{2}{c}{\textbf{AIME 2025}} &
    \multicolumn{2}{c}{\textbf{HMMT 2025}} \\
\cmidrule(r){2-3}\cmidrule(r){4-5}\cmidrule(l){6-7}
& \textbf{Acc.}$\uparrow$ & \textbf{\#Tok.}$\downarrow$
& \textbf{Acc.}$\uparrow$ & \textbf{\#Tok.}$\downarrow$
& \textbf{Acc.}$\uparrow$ & \textbf{\#Tok.}$\downarrow$ \\
\midrule
Base & 77.9 & 19,442 & 73.9 & 21,170 & 51.7 & 23,906 \\
Klear-only & 83.0 & 18,128 & 78.5 & 19,792 & 56.1 & 23,377 \\
DECS-only & 79.0 & 17,680 & 75.5 & 19,460 & 50.4 & 21,369 \\
\midrule
\multicolumn{7}{l}{\textit{Klear--DECS composition}} \\
\quad $\rho_{\mathrm{DECS}}=0.250$ & 83.8 & 17,474 & 78.5 & 19,300 & 55.6 & 22,774 \\
\quad $\rho_{\mathrm{DECS}}=0.375$ & 85.1 & 17,038 & 79.3 & 18,760 & 56.6 & 22,197 \\
\quad $\rho_{\mathrm{DECS}}=0.500$ & 85.3 & 16,003 & 80.4 & 17,970 & 54.4 & 20,440 \\
\quad $\rho_{\mathrm{DECS}}=0.625$ & 84.9 & 15,048 & 80.4 & 16,963 & 52.9 & 18,470 \\
\quad $\rho_{\mathrm{DECS}}=0.750$ & 83.3 & 14,310 & 79.6 & 16,409 & 50.9 & 16,147 \\
\midrule
\multicolumn{7}{l}{\textit{Additional baselines}} \\
SimpleSD & 76.7 & 19,473 & 75.2 & 21,157 & 51.0 & 24,053 \\
Art & 78.5 & 15,456 & 76.3 & 17,703 & 49.4 & 17,339 \\
PromptCoT 2.0 & 75.8 & 20,177 & 76.9 & 21,455 & 54.0 & 24,838 \\
MI ($\lambda=0.8$) & 83.1 & 15,885 & 77.5 & 17,542 & 52.3 & 19,653 \\
\bottomrule
\end{tabular}%
}
\vspace{-1.0\baselineskip}
\end{wraptable}

Responses generally
shorten as the DECS fraction increases, while intermediate compositions preserve
or improve accuracy. On AIME 2024, for example,
$\rho_{\mathrm{DECS}}=0.625$ achieves 84.9\% accuracy with 15,048 tokens,
exceeding every external baseline and both single-anchor policies in accuracy
while using fewer tokens. The non-dominated \method policies consequently
trace an empirical Pareto frontier in
Figure~\ref{fig:accuracy-efficiency-tradeoff}(b).
In Figure~\ref{fig:accuracy-efficiency-tradeoff}(c), we evaluate AIME 2024
accuracy under different token budgets. \method outperforms both single-anchor
policies at every budget, with the largest gains under tight budgets.
Together, these results establish a controllable
accuracy--efficiency frontier whose gains persist across metrics and inference
budgets.

\begin{figure*}[!t]
\centering
\begin{minipage}[t]{0.27\textwidth}
\vspace{0pt}
\centering
\includegraphics[width=\linewidth]{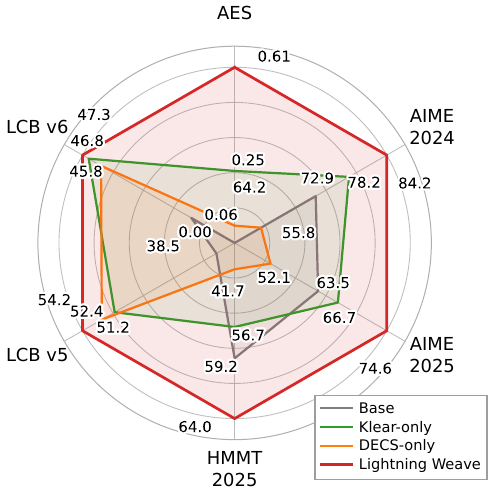}
\par\textnormal{(a)}
\end{minipage}\hfill
\begin{minipage}[t]{0.388\textwidth}
\vspace{0pt}
\centering
\includegraphics[width=\linewidth]{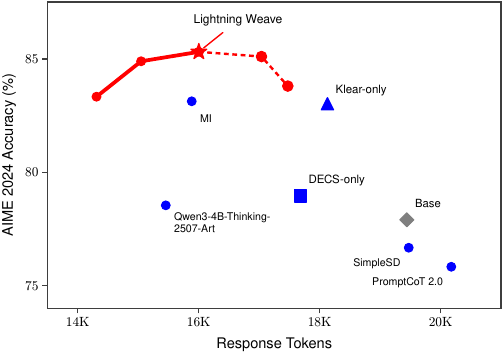}
\par\textnormal{(b)}
\end{minipage}\hfill
\begin{minipage}[t]{0.326\textwidth}
\vspace{0pt}
\centering
\includegraphics[width=\linewidth]{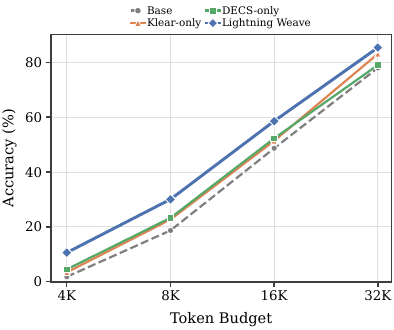}
\par\textnormal{(c)}
\end{minipage}
\caption{(a): The performance of \method and baseline policies on Qwen3.5-4B
across AES and all five benchmarks. \method achieves the strongest overall
capability profile. (b): Accuracy--efficiency trade-off of \method and the
baselines on AIME 2024 with Qwen3-4B-Thinking-2507. Varying the anchor
weights yields a family of \method policies whose non-dominated operating
points trace a Pareto frontier against the efficient-reasoning baselines. The
solid red line connects these Pareto-optimal points, the dashed red line
connects the remaining weight-sweep settings, and the star marks the
highest-accuracy setting. (c): The
performance of the Base, Klear-only, DECS-only, and
\method policies on AIME 2024 under different token budgets.
\method performs best at every budget.}
\label{fig:accuracy-efficiency-tradeoff}
\end{figure*}

\subsection{Generalization to Additional Anchor Combinations}
\label{sec:anchor-generalization}

\begin{table}[!htbp]
\centering
\caption{Generalization to additional anchor combinations on Qwen3-4B. We
report accuracy (Acc., \%) and mean response tokens (\#Tok.) for each
benchmark, together with average accuracy, response tokens, and AES.
Multi-anchor compositions use equal normalized weights. The best value for each metric
is marked in \textbf{bold} and the second-best is \underline{underlined}.}
\label{tab:anchor-generalization}
\resizebox{\textwidth}{!}{%
\begin{tabular}{lccccccccccccc}
\toprule
\multirow{2}[4]{*}{\textbf{Policy}} &
    \multicolumn{2}{c}{\textbf{AIME 2024}} &
    \multicolumn{2}{c}{\textbf{AIME 2025}} &
    \multicolumn{2}{c}{\textbf{HMMT 2025}} &
    \multicolumn{2}{c}{\textbf{LCB v5}} &
    \multicolumn{2}{c}{\textbf{LCB v6}} &
    \multicolumn{3}{c}{\textbf{Average}} \\
\cmidrule(r){2-3}\cmidrule(r){4-5}\cmidrule(r){6-7}\cmidrule(r){8-9}\cmidrule(r){10-11}\cmidrule(l){12-14}
& \textbf{Acc.}$\uparrow$ & \textbf{\#Tok.}$\downarrow$
& \textbf{Acc.}$\uparrow$ & \textbf{\#Tok.}$\downarrow$
& \textbf{Acc.}$\uparrow$ & \textbf{\#Tok.}$\downarrow$
& \textbf{Acc.}$\uparrow$ & \textbf{\#Tok.}$\downarrow$
& \textbf{Acc.}$\uparrow$ & \textbf{\#Tok.}$\downarrow$
& \textbf{Acc.}$\uparrow$ & \textbf{\#Tok.}$\downarrow$ & \textbf{AES}$\uparrow$ \\
\midrule
Base
    & 72.9 & 14,641 & 65.4 & 17,844
    & 41.4 & 17,626 & 53.0 & 13,843
    & 46.8 & 14,663
    & 55.90 & 15,723 & 0.00 \\
\midrule
\multicolumn{14}{l}{\textbf{\textit{Single anchor}}} \\
Klear-only
    & 75.8 & 15,061 & 71.3 & 17,715
    & 45.6 & 18,230 & 57.4 & 13,252
    & 49.6 & 13,776
    & 59.94 & 15,607 & 0.23 \\
DECS-only
    & 68.3 & \textbf{10,918} & 60.6 & \textbf{12,429}
    & 35.0 & \textbf{11,862} & 53.0 & \underline{12,531}
    & 46.2 & 13,437
    & 52.62 & \textbf{12,235} & $-0.09$ \\
MiMo-RL
    & \underline{76.0} & 16,646 & 71.3 & 19,580
    & 45.8 & 19,879 & 57.0 & 14,383
    & 47.9 & 14,939
    & 59.60 & 17,085 & 0.12 \\
\midrule
\multicolumn{14}{l}{\textbf{\textit{Two anchors}}} \\
Klear+DECS
    & \underline{76.0} & \underline{11,521}
    & 70.1 & \underline{13,816}
    & 43.0 & \underline{13,743}
    & \underline{59.2} & \textbf{11,702}
    & 51.3 & \textbf{12,357}
    & 59.92 & \underline{12,628} & \underline{0.41} \\
Klear+MiMo-RL
    & \textbf{78.7} & 19,031 & \textbf{74.9} & 21,355
    & \textbf{50.5} & 22,328 & \textbf{63.4} & 14,778
    & \textbf{54.4} & 15,395
    & \textbf{64.38} & 18,577 & 0.31 \\
DECS+MiMo-RL
    & 74.2 & 12,966 & 66.5 & 15,336
    & 40.9 & 15,144 & 56.5 & 12,915
    & 48.3 & \underline{13,333}
    & 57.28 & 13,939 & 0.18 \\
\midrule
\multicolumn{14}{l}{\textbf{\textit{Three anchors}}} \\
\rowcolor[rgb]{.867,.922,.969}
Klear+DECS+MiMo-RL
    & \textbf{78.7} & 14,268
    & \underline{73.6} & 16,838
    & \underline{47.5} & 16,731
    & \underline{59.2} & 12,930
    & \underline{53.4} & 13,609
    & \underline{62.50} & 14,875 & \textbf{0.42} \\
\bottomrule
\end{tabular}%
}
\end{table}

We further introduce MiMo-RL~\citep{xiaomi2025mimo} as an additional
accuracy-oriented anchor and compose it with Klear and DECS, using Qwen3-4B
as the student model. As shown in Table~\ref{tab:anchor-generalization},
combining the two accuracy anchors, Klear+MiMo-RL, achieves the highest
overall accuracy but also substantially increases the number of inference
tokens. In contrast, DECS+MiMo-RL shows the same qualitative pattern as
Klear+DECS, improving accuracy over the base model while reducing response
tokens. We further extend \method to a three-anchor composition.
Relative to each of its two-anchor subsets, adding the third anchor shifts the
operating point toward the corresponding capability: adding DECS to
Klear+MiMo-RL improves efficiency, whereas adding Klear or MiMo-RL to the
corresponding DECS pair improves accuracy. The resulting composition balances
all three shifts and achieves the highest AES of 0.42.

\subsection{Ablation Study}
\label{sec:ablations}

\paragraph{Tilted-target correction.}
We compare Tilted-Target with Online DOPD, which recomputes each policy shift
on current-student trajectories, and Naive Offline, which directly optimizes
the cached shift with a behavior-sampled KL penalty. Within each of the Klear,
DECS, and Klear+DECS settings, all methods use matched student initializations,
prompt sequences, trajectory budgets, and update budgets. As shown in
Table~\ref{tab:ablation-summary}(a), Tilted-Target performs comparably to
Online DOPD across all three settings. In the two-anchor setting, however, it
avoids concurrently serving the four models that constitute the anchor pairs
during training, making the overall training pipeline substantially more
accessible. It also consistently yields a better trade-off than Naive Offline:
average accuracy improves by 1.51 and 3.85 points with Klear and DECS,
respectively, while the two-anchor setting uses 17.6\% fewer tokens with
slightly higher accuracy. These results show that the explicit target
stabilizes cached-shift optimization while retaining the benefits of online
supervision.

\paragraph{Multi-anchor composition.}
We compare \method with Data Mixture, which interleaves equal amounts of data
scored separately by Klear and DECS, and Sequential Composition, which applies
the two single-anchor shifts in successive stages in either order. As shown in
Table~\ref{tab:ablation-summary}(b), \method improves accuracy over Data
Mixture on all five benchmarks while using 3.8--8.7\% fewer tokens.
Sequential Composition is strongly order-sensitive, with the last applied
anchor determining whether the result favors accuracy or efficiency. By
avoiding both mixture dilution and ordering bias, \method achieves the best
overall accuracy--efficiency trade-off, as reflected by the highest AES of
0.41.

\begin{table}[!htbp]
\centering
\caption{Policy-shift distillation and multi-anchor composition ablations on
Qwen3-4B across mathematics and code. The upper block compares Online DOPD,
Naive Offline, and Tilted-Target under Klear-only, DECS-only, and their
composition. Tilted-Target closely matches Online DOPD while improving the
accuracy--efficiency trade-off over Naive Offline. The lower block compares
\method with Data Mixture and both sequential orders under equal anchor
weights and matched trajectory and update budgets. We report accuracy and mean
response tokens for each benchmark,
together with five-benchmark averages and AES. The best and second-best values
within each comparison are marked in \textbf{bold} and
\underline{underlined}, respectively.}
\vspace{\baselineskip}
\label{tab:ablation-summary}
\resizebox{\textwidth}{!}{%
\begin{tabular}{lccccccccccccc}
\toprule
\multirow{2}[4]{*}{\textbf{Variant}} &
    \multicolumn{2}{c}{\textbf{AIME 2024}} &
    \multicolumn{2}{c}{\textbf{AIME 2025}} &
    \multicolumn{2}{c}{\textbf{HMMT 2025}} &
    \multicolumn{2}{c}{\textbf{LCB v5}} &
    \multicolumn{2}{c}{\textbf{LCB v6}} &
    \multicolumn{3}{c}{\textbf{Average}} \\
\cmidrule(r){2-3}\cmidrule(r){4-5}\cmidrule(r){6-7}\cmidrule(r){8-9}\cmidrule(r){10-11}\cmidrule(l){12-14}
& \textbf{Acc.}$\uparrow$ & \textbf{\#Tok.}$\downarrow$
& \textbf{Acc.}$\uparrow$ & \textbf{\#Tok.}$\downarrow$
& \textbf{Acc.}$\uparrow$ & \textbf{\#Tok.}$\downarrow$
& \textbf{Acc.}$\uparrow$ & \textbf{\#Tok.}$\downarrow$
& \textbf{Acc.}$\uparrow$ & \textbf{\#Tok.}$\downarrow$
& \textbf{Acc.}$\uparrow$ & \textbf{\#Tok.}$\downarrow$
& \textbf{AES}$\uparrow$ \\
\midrule
Base
    & 72.9 & 14,641
    & 65.4 & 17,844
    & 41.4 & 17,626
    & 53.0 & 13,843
    & 46.8 & 14,663
    & 55.90 & 15,723 & 0.00 \\
\midrule
\multicolumn{14}{l}{\textbf{\textit{(a) Tilted-target correction}}} \\
\midrule
\multicolumn{14}{l}{\textbf{\textit{Single anchor: Qwen3-4B + Klear-only}}} \\
Online DOPD
    & \textbf{76.4} & \textbf{14,673}
    & \underline{69.8} & \underline{17,903}
    & \underline{45.3} & \underline{18,447}
    & \underline{57.7} & \underline{13,950}
    & \underline{50.6} & \underline{14,461}
    & \textbf{59.96} & \underline{15,887} & \underline{0.21} \\
Naive Offline
    & 70.7 & 21,508
    & 65.8 & 23,671
    & \textbf{45.6} & 24,401
    & \textbf{58.2} & 15,549
    & \textbf{51.7} & 16,316
    & 58.43 & 20,289 & $-0.15$ \\
\rowcolor[rgb]{.867,.922,.969}
Tilted-Target (ours)
    & \underline{75.8} & \underline{15,061}
    & \textbf{71.3} & \textbf{17,715}
    & \textbf{45.6} & \textbf{18,230}
    & 57.4 & \textbf{13,252}
    & 49.6 & \textbf{13,776}
    & \underline{59.94} & \textbf{15,607} & \textbf{0.23} \\
\midrule
\multicolumn{14}{l}{\textbf{\textit{Single anchor: Qwen3-4B + DECS-only}}} \\
Online DOPD
    & \underline{68.2} & \underline{10,240}
    & \underline{59.0} & \underline{11,942}
    & \underline{34.9} & \underline{11,663}
    & \underline{52.0} & \underline{12,017}
    & \underline{45.2} & \underline{12,832}
    & \underline{51.87} & \underline{11,739} & $\underline{-0.11}$ \\
Naive Offline
    & 64.5 & \textbf{9,472}
    & 55.0 & \textbf{10,796}
    & 32.1 & \textbf{10,065}
    & 49.3 & \textbf{11,071}
    & 42.9 & \textbf{11,738}
    & 48.77 & \textbf{10,628} & $-0.31$ \\
\rowcolor[rgb]{.867,.922,.969}
Tilted-Target (ours)
    & \textbf{68.3} & 10,918
    & \textbf{60.6} & 12,429
    & \textbf{35.0} & 11,862
    & \textbf{53.0} & 12,531
    & \textbf{46.2} & 13,437
    & \textbf{52.62} & 12,235 & $\mathbf{-0.09}$ \\
\midrule
\multicolumn{14}{l}{\textbf{\textit{Two anchors: Qwen3-4B + Klear+DECS}}} \\
Online DOPD
    & \textbf{76.1} & \underline{12,591}
    & \underline{70.9} & \underline{15,129}
    & \textbf{45.6} & \underline{15,454}
    & \textbf{59.8} & \underline{12,921}
    & \textbf{52.3} & \underline{13,582}
    & \textbf{60.93} & \underline{13,935} & \underline{0.38} \\
Naive Offline
    & 74.6 & 15,352
    & \textbf{72.0} & 17,468
    & \underline{44.2} & 16,656
    & 58.3 & 13,338
    & 49.6 & 13,771
    & 59.76 & 15,317 & 0.24 \\
\rowcolor[rgb]{.867,.922,.969}
Tilted-Target (ours)
    & \underline{76.0} & \textbf{11,521}
    & 70.1 & \textbf{13,816}
    & 43.0 & \textbf{13,743}
    & \underline{59.2} & \textbf{11,702}
    & \underline{51.3} & \textbf{12,357}
    & \underline{59.92} & \textbf{12,628} & \textbf{0.41} \\
\midrule
\multicolumn{14}{l}{\textbf{\textit{(b) Multi-anchor composition}}} \\
\midrule
Data Mixture
    & 75.4 & 12,579
    & 68.7 & 14,965
    & 42.9 & 15,059
    & 55.4 & 12,551
    & 51.0 & 12,840
    & 58.68 & 13,599 & 0.29 \\
Klear-only $\rightarrow$ DECS-only
    & 74.3 & \textbf{11,389}
    & 65.8 & \textbf{13,702}
    & 41.4 & \textbf{13,474}
    & 56.5 & \textbf{11,609}
    & 50.8 & \textbf{12,160}
    & 57.76 & \textbf{12,467} & 0.31 \\
DECS-only $\rightarrow$ Klear-only
    & \textbf{77.4} & 13,337
    & \textbf{71.8} & 16,088
    & \textbf{46.0} & 16,551
    & \underline{57.7} & 12,555
    & \textbf{51.5} & 12,959
    & \textbf{60.88} & 14,298 & \underline{0.37} \\
\rowcolor[rgb]{.867,.922,.969}
\method
    & \underline{76.0} & \underline{11,521}
    & \underline{70.1} & \underline{13,816}
    & \underline{43.0} & \underline{13,743}
    & \textbf{59.2} & \underline{11,702}
    & \underline{51.3} & \underline{12,357}
    & \underline{59.92} & \underline{12,628} & \textbf{0.41} \\
\bottomrule
\end{tabular}%
}
\end{table}

\section{Conclusion}

In this work, we introduced capability composition as a new route to efficient
reasoning.
Rather than jointly optimizing competing accuracy and efficiency objectives
from scratch, our approach reuses independently post-trained specialists,
extracts the behavioral shift acquired by each, and composes these shifts into
one student model. We instantiate this idea with \method, which aligns multiple
policy shifts at shared student token states and combines them into a joint
tilted target. Tilted-Target DOPD converts cached shifts into explicit targets
with well-defined fixed points, enabling stable offline training without
concurrently serving live anchor models. Across diverse student models and
benchmarks spanning mathematics and code, \method consistently improves the
accuracy--efficiency trade-off over the base and single-anchor policies,
generalizes to additional anchor combinations, and yields a controllable
empirical Pareto frontier as the composition weights vary. These results
demonstrate the value of \method as a new approach to efficient reasoning,
improving the accuracy--efficiency frontier by composing capabilities learned
independently for distinct objectives.

\bibliography{iclr2027_conference}
\bibliographystyle{iclr2027_conference}

\raggedbottom
\clearpage
\appendix
\setcounter{topnumber}{4}
\setcounter{bottomnumber}{3}
\setcounter{totalnumber}{6}
\renewcommand{\topfraction}{0.98}
\renewcommand{\bottomfraction}{0.95}
\renewcommand{\textfraction}{0.02}
\renewcommand{\floatpagefraction}{0.80}
\setlength{\textfloatsep}{10pt plus 2pt minus 2pt}
\setlength{\floatsep}{8pt plus 2pt minus 2pt}
\setlength{\intextsep}{8pt plus 2pt minus 2pt}

\section{Theoretical Analysis and Proofs}
\label{app:theory}

This section analyzes the full-distribution formulation underlying
Tilted-Target DOPD and \method. We characterize the explicit learning target
for offline capability transfer and the joint target obtained by composing
multiple policy shifts. The results are stated for the zero-discount,
token-level DOPD surrogate used in Section~\ref{sec:method}: at each fixed
state, the anchor shift is treated as an immediate reward. They do not identify
the resulting local target with the conditionals of the idealized
sequence-level optimum in Eq.~\ref{eq:dopd-objective}, whose conditionals can
also depend on future shifts. We first derive a general variational identity,
then apply it to one anchor pair and to a weighted composition of multiple
anchor-pair shifts.

\subsection{Setup and Assumptions}
\label{app:theory-setup}

Let $\mathcal{A}$ be the finite student-token action space and
$\mathcal{P}(\mathcal{A})$ its probability simplex. Fix a token-level state
$s$ in the support of the cached state distribution $d_{\pi_b}$, and write
$b_s(a)=\pi_b(a\mid s)$. We use $d_{\pi_b}$ for the behavior-policy state
distribution. The implementation replaces it by the corresponding finite
empirical distribution of cached states. The target construction is
pointwise in $s$. We make the following assumptions:
\begin{enumerate}
    \item $\alpha>0$ and $b_s(a)>0$ for every $a\in\mathcal{A}$,
    \item every aligned anchor score $r_i(a\mid s)$ is finite, and
    \item $b_s$ and the anchor scores are cached quantities and therefore do
    not depend on the optimized parameters $\theta$.
\end{enumerate}
These conditions hold for finite-logit softmax policies in exact arithmetic.
More generally, the analysis can be restricted to any common support on which
the conditions hold.
For extensions to a non-empirical state distribution, we additionally assume
that the reward, KL, and log-partition terms have finite expectations for the
policies under consideration. This condition holds automatically for a
finite cache under the assumptions above.

For a finite score function $r:\mathcal{A}\rightarrow\mathbb{R}$ and a
distribution $p\in\mathcal{P}(\mathcal{A})$, define the local regularized
objective
\begin{equation}
    \Phi_s(p,r)
    =
    \sum_{a\in\mathcal{A}}p(a)r(a\mid s)
    -
    \alpha D_{\mathrm{KL}}(p\,\|\,b_s),
    \label{eq:app-local-objective}
\end{equation}
where $0\log 0=0$. Its partition function and exponentially tilted
distribution are
\begin{align}
    Z_r(s)
    &=
    \sum_{a\in\mathcal{A}}
    b_s(a)\exp\!\left(\frac{r(a\mid s)}{\alpha}\right),
    \label{eq:app-partition-function}\\
    T_r(a\mid s)
    &=
    \frac{
        b_s(a)\exp\!\left(r(a\mid s)/\alpha\right)
    }{Z_r(s)}.
    \label{eq:app-general-tilt}
\end{align}
Finiteness of $\mathcal{A}$ and of $r$ implies
$0<Z_r(s)<\infty$, so $T_r(\cdot\mid s)$ is a well-defined,
strictly positive distribution.

When all policies share an action space and both anchors are strictly positive
on it, an anchor pair induces
\begin{equation}
    r_i(a\mid s)
    =
    \log\frac{\pi_i^+(a\mid s)}{\pi_i^-(a\mid s)}.
    \label{eq:app-exact-anchor-shift}
\end{equation}
The variational results below require only that $r_i$ be finite. Thus they
also apply to the finite scores produced by cross-tokenizer alignment. The
exact product-of-policy-ratios interpretation additionally requires
Eq.~\ref{eq:app-exact-anchor-shift}. We describe the practical alignment
procedure separately in Appendix~\ref{app:cross-tokenizer}.

\subsection{Tilted-Target DOPD: Objective Equivalence and Fixed Point}
\label{app:tilted-target-proof}

\begin{lemma}[Exponential-tilting identity]
\label{lem:tilting-identity}
For every $p\in\mathcal{P}(\mathcal{A})$,
\begin{equation}
    \alpha D_{\mathrm{KL}}(p\,\|\,T_r(\cdot\mid s))
    =
    \alpha D_{\mathrm{KL}}(p\,\|\,b_s)
    -
    \sum_{a\in\mathcal{A}}p(a)r(a\mid s)
    +
    \alpha\log Z_r(s).
    \label{eq:app-tilting-identity}
\end{equation}
\end{lemma}

\begin{proof}
First suppose $p(a)>0$ for every $a$.
Equation~\ref{eq:app-general-tilt} gives
\[
    \log T_r(a\mid s)
    =
    \log b_s(a)
    +
    \frac{r(a\mid s)}{\alpha}
    -
    \log Z_r(s).
\]
Substituting this expression into the definition of KL divergence yields
\begin{align*}
    \alpha D_{\mathrm{KL}}(p\,\|\,T_r)
    &=
    \alpha\sum_a p(a)
    \left[
        \log p(a)-\log b_s(a)
        -\frac{r(a\mid s)}{\alpha}
        +\log Z_r(s)
    \right]\\
    &=
    \alpha D_{\mathrm{KL}}(p\,\|\,b_s)
    -
    \sum_a p(a)r(a\mid s)
    +
    \alpha\log Z_r(s),
\end{align*}
where the last term uses $\sum_a p(a)=1$.
For a distribution on the boundary of the simplex, the same identity follows
by continuity using $\lim_{u\downarrow 0}u\log u=0$, since both $b_s$ and
$T_r$ are strictly positive.
\end{proof}

\begin{proposition}[Unique single-anchor target]
\label{prop:single-anchor-target}
For a fixed state $s$, the tilted distribution $T_r(\cdot\mid s)$ is the
unique maximizer of the local objective in
Eq.~\ref{eq:app-local-objective}. In particular, taking
$r=r_\Delta$ recovers the target $q$ in Eq.~\ref{eq:tilted-target}.
\end{proposition}

\begin{proof}
Rearranging Lemma~\ref{lem:tilting-identity} gives
\begin{equation}
    \Phi_s(p,r)
    =
    \alpha\log Z_r(s)
    -
    \alpha D_{\mathrm{KL}}(p\,\|\,T_r(\cdot\mid s)).
    \label{eq:app-local-objective-as-kl}
\end{equation}
Gibbs' inequality gives
$D_{\mathrm{KL}}(p\,\|\,T_r)\geq 0$, with equality if and only if
$p=T_r$ because $T_r$ is strictly positive. Since $\alpha>0$, the stated
maximizer is therefore unique.
\end{proof}

The pointwise result extends directly to the fixed offline state
distribution. Define
\begin{align}
    \mathcal{J}_{d}(\pi,r)
    &=
    \mathbb{E}_{s\sim d_{\pi_b}}
    \left[\Phi_s(\pi(\cdot\mid s),r)\right],
    \label{eq:app-offline-objective}\\
    \mathcal{L}_{d}(\pi,r)
    &=
    \alpha\,
    \mathbb{E}_{s\sim d_{\pi_b}}
    \left[
        D_{\mathrm{KL}}
        \bigl(
            \pi(\cdot\mid s)\,\|\,T_r(\cdot\mid s)
        \bigr)
    \right].
    \label{eq:app-offline-target-loss}
\end{align}

\begin{corollary}[Offline objective equivalence]
\label{cor:offline-objective-equivalence}
For every policy $\pi$,
\begin{equation}
    \mathcal{L}_{d}(\pi,r)
    =
    -\mathcal{J}_{d}(\pi,r)
    +
    \alpha\,
    \mathbb{E}_{s\sim d_{\pi_b}}
    [\log Z_r(s)].
    \label{eq:app-offline-equivalence}
\end{equation}
Consequently, minimizing $\mathcal{L}_d$ and maximizing $\mathcal{J}_d$ are
equivalent over any common policy class, including a restricted parametric
class. If that class can realize $T_r$ on
$d_{\pi_b}$-almost every state, then $T_r$ is a global minimizer with zero
loss. Otherwise, training minimizes the expected reverse KL to this explicit
target.
\end{corollary}

\begin{proof}
Take the expectation of Eq.~\ref{eq:app-tilting-identity} over the fixed
distribution $d_{\pi_b}$. The final term in
Eq.~\ref{eq:app-offline-equivalence} depends only on cached quantities, not
on $\pi$. The statements about optimization over a policy class follow
immediately.
\end{proof}

\begin{corollary}[Fixed-point property]
\label{cor:fixed-point-property}
Suppose $\pi_{\theta^\star}(\cdot\mid s)=T_r(\cdot\mid s)$ for
$d_{\pi_b}$-almost every $s$, and $\pi_\theta$ is differentiable at
$\theta^\star$. For a non-empirical state distribution, additionally assume
that differentiation and expectation can be interchanged. Then
\[
    \mathcal{L}_{d}(\pi_{\theta^\star},r)=0
    \qquad\text{and}\qquad
    \nabla_\theta
    \mathcal{L}_{d}(\pi_\theta,r)
    \big|_{\theta=\theta^\star}
    =0.
\]
\end{corollary}

\begin{proof}
The first statement follows from equality of the two distributions. For a
fixed state, differentiating the KL divergence at
$\pi_{\theta^\star}=T_r$ gives
\begin{align*}
    \nabla_\theta
    D_{\mathrm{KL}}(\pi_\theta\,\|\,T_r)
    \big|_{\theta=\theta^\star}
    &=
    \sum_a
    \nabla_\theta\pi_{\theta^\star}(a\mid s)
    \left(
        \log
        \frac{\pi_{\theta^\star}(a\mid s)}{T_r(a\mid s)}
        +1
    \right)\\
    &=
    \nabla_\theta
    \sum_a\pi_{\theta^\star}(a\mid s)
    =0.
\end{align*}
Taking the expectation over fixed cached states preserves the equality.
\end{proof}

Corollary~\ref{cor:offline-objective-equivalence} also precisely states the
connection to the online update at initialization. If
$\pi_{\theta_0}=\pi_b$, then the online state distribution and the
distribution used to generate the cache coincide at $\theta_0$. Hence, in
expectation over cache construction and when sampled states are treated as
fixed during the actor update, gradient descent on the Tilted-Target loss and
gradient ascent on the zero-discount DOPD surrogate have the same expected
per-state update direction at that point. A particular finite cache retains
ordinary sampling error. The statement also does not extend to later online
state distributions once $\pi_\theta$ departs from $\pi_b$.

\paragraph{Why the cached sampled-KL surrogate lacks this guarantee.}
The sampled KL estimator in Eq.~\ref{eq:naive-offline-dopd} is
$g(\ell)=\exp(\ell)-1-\ell$, where
$\ell(a)=\log b_s(a)-\log p(a)$. For a strictly positive $p$, if
$a\sim p$, then
\begin{equation}
    \mathbb{E}_{a\sim p}[g(\ell(a))]
    =
    \sum_a p(a)
    \left[
        \frac{b_s(a)}{p(a)}
        -1
        -\log\frac{b_s(a)}{p(a)}
    \right]
    =
    D_{\mathrm{KL}}(p\,\|\,b_s).
    \label{eq:app-on-policy-sampled-kl}
\end{equation}
For a cached action $a_b\sim b_s$, however, its population expectation is
\begin{equation}
    \widetilde{K}_{b_s}(p)
    =
    \sum_a b_s(a)
    \left[
        \frac{b_s(a)}{p(a)}
        -1
        -\log\frac{b_s(a)}{p(a)}
    \right],
    \label{eq:app-cached-sampled-kl}
\end{equation}
which is generally not $D_{\mathrm{KL}}(p\,\|\,b_s)$. Therefore the
identity in Lemma~\ref{lem:tilting-identity} does not apply to the expected
cached sampled-KL objective away from $p=b_s$.

For a concrete counterexample, take two actions and set $\alpha=1$,
$b_s=\pi^-=(1/2,1/2)$, and $\pi^+=(3/4,1/4)$. This anchor pair induces
$r=(\log(3/2),\log(1/2))$, and the exact tilted target is
$T_r=(3/4,1/4)$. Parameterizing $p=(x,1-x)$, direct differentiation of
Eq.~\ref{eq:app-cached-sampled-kl} gives
\[
    \frac{\mathrm{d}}{\mathrm{d}x}
    \widetilde{K}_{b_s}(x,1-x)
    =
    \frac{1}{2}\frac{x-1/2}{x^2}
    -
    \frac{1}{2}
    \frac{(1-x)-1/2}{(1-x)^2},
\]
and therefore
\[
    \left.
    \frac{\mathrm{d}}{\mathrm{d}x}
    \widetilde{K}_{b_s}(x,1-x)
    \right|_{x=3/4}
    =
    \frac{20}{9}.
\]
Thus the derivative of
$\Phi_s^{\mathrm{cached}}(p,r)
=\langle p,r\rangle-\alpha\widetilde{K}_{b_s}(p)$ at the exact target is
$\log 3-20/9\neq 0$. The intended tilted target is not even a
stationary point in this example. This establishes a lack of a general
fixed-point guarantee. It does not assert that every finite run of the naive
surrogate must fail.

\subsection{\method: Joint Targets for Capability Composition}
\label{app:multi-anchor-proof}

For $K$ anchor pairs, let $\mathbf{w}=(w_1,\ldots,w_K)$ with
$w_i\geq 0$, and define the joint score
\[
    r_{\mathbf{w}}(a\mid s)
    =
    \sum_{i=1}^{K}w_i r_i(a\mid s).
\]

\begin{proposition}[Unique multi-anchor target]
\label{prop:multi-anchor-target}
For every fixed state $s$, the unique solution to
\begin{equation}
    \max_{p\in\mathcal{P}(\mathcal{A})}
    \left\{
        \sum_{i=1}^{K}w_i
        \mathbb{E}_{a\sim p}[r_i(a\mid s)]
        -
        \alpha D_{\mathrm{KL}}(p\,\|\,b_s)
    \right\}
    \label{eq:app-multi-anchor-local-objective}
\end{equation}
is
\begin{equation}
    q_{\mathbf{w}}(a\mid s)
    =
    \frac{
        b_s(a)
        \exp\!\left(
            \frac{1}{\alpha}
            \sum_{i=1}^{K}w_i r_i(a\mid s)
        \right)
    }{Z_{\mathbf{w}}(s)},
    \label{eq:app-multi-anchor-target}
\end{equation}
where
\[
    Z_{\mathbf{w}}(s)
    =
    \sum_a b_s(a)
    \exp\!\left(
        \frac{1}{\alpha}
        \sum_{i=1}^{K}w_i r_i(a\mid s)
    \right).
\]
If every score satisfies Eq.~\ref{eq:app-exact-anchor-shift}, then
\begin{equation}
    q_{\mathbf{w}}(a\mid s)
    =
    \frac{b_s(a)}{Z_{\mathbf{w}}(s)}
    \prod_{i=1}^{K}
    \left(
        \frac{\pi_i^+(a\mid s)}
             {\pi_i^-(a\mid s)}
    \right)^{w_i/\alpha}.
    \label{eq:app-multi-anchor-product}
\end{equation}
\end{proposition}

\begin{proof}
Linearity of expectation makes
Eq.~\ref{eq:app-multi-anchor-local-objective} equal to
$\Phi_s(p,r_{\mathbf{w}})$. Proposition~\ref{prop:single-anchor-target},
applied to the finite score $r_{\mathbf{w}}$, gives the unique solution in
Eq.~\ref{eq:app-multi-anchor-target}. Under
Eq.~\ref{eq:app-exact-anchor-shift},
\begin{align*}
    \exp\!\left(
        \frac{1}{\alpha}
        \sum_i w_i r_i(a\mid s)
    \right)
    &=
    \exp\!\left(
        \sum_i\frac{w_i}{\alpha}
        \log\frac{\pi_i^+(a\mid s)}{\pi_i^-(a\mid s)}
    \right)\\
    &=
    \prod_i
    \left(
        \frac{\pi_i^+(a\mid s)}{\pi_i^-(a\mid s)}
    \right)^{w_i/\alpha},
\end{align*}
which proves Eq.~\ref{eq:app-multi-anchor-product}.
\end{proof}

Define the fixed-state-distribution multi-anchor objective
\begin{equation}
    \mathcal{J}_{\mathrm{joint},d}(\pi,\mathbf{w})
    =
    \mathbb{E}_{s\sim d_{\pi_b}}
    \left[
        \sum_{i=1}^{K}w_i
        \mathbb{E}_{a\sim\pi(\cdot\mid s)}
        [r_i(a\mid s)]
        -
        \alpha D_{\mathrm{KL}}
        \bigl(
            \pi(\cdot\mid s)\,\|\,b_s
        \bigr)
    \right].
    \label{eq:app-maps-offline-objective}
\end{equation}

\begin{corollary}[\method objective equivalence]
\label{cor:maps-objective-equivalence}
For every policy $\pi$,
\begin{align}
    &\alpha
    \mathbb{E}_{s\sim d_{\pi_b}}
    \left[
        D_{\mathrm{KL}}
        \bigl(
            \pi(\cdot\mid s)\,\|\,q_{\mathbf{w}}(\cdot\mid s)
        \bigr)
    \right]
    \nonumber\\
    &\qquad
    =
    -\mathcal{J}_{\mathrm{joint},d}(\pi,\mathbf{w})
    +
    \alpha
    \mathbb{E}_{s\sim d_{\pi_b}}
    [\log Z_{\mathbf{w}}(s)].
    \label{eq:app-maps-objective-equivalence}
\end{align}
Hence the \method loss in Eq.~\ref{eq:maps-loss} is exactly the negative
joint regularized objective on the fixed cached states, up to a constant
independent of the student policy. This equivalence holds over any common
policy class.
\end{corollary}

\begin{proof}
Apply Lemma~\ref{lem:tilting-identity} with
$r=r_{\mathbf{w}}$ and take the expectation over $d_{\pi_b}$.
\end{proof}

\begin{corollary}[Reduction to Tilted-Target DOPD]
\label{cor:single-anchor-reduction}
If $K=1$ and $w_1=1$, then
$r_{\mathbf{w}}=r_1$,
$q_{\mathbf{w}}=T_{r_1}$, and
Eq.~\ref{eq:app-maps-objective-equivalence} reduces exactly to
Corollary~\ref{cor:offline-objective-equivalence}. If all $w_i=0$, then
$q_{\mathbf{w}}=b_s$.
\end{corollary}

\begin{proof}
Both statements follow by substitution into
Eq.~\ref{eq:app-multi-anchor-target}.
\end{proof}

\subsection{Structural Properties of the Joint Target}
\label{app:maps-properties}

\begin{proposition}[Pairwise-odds decomposition]
\label{prop:pairwise-odds}
For any two actions $a,a'\in\mathcal{A}$,
\begin{equation}
    \log
    \frac{q_{\mathbf{w}}(a\mid s)}
         {q_{\mathbf{w}}(a'\mid s)}
    =
    \log\frac{b_s(a)}{b_s(a')}
    +
    \frac{1}{\alpha}
    \sum_{i=1}^{K}w_i
    \left[
        r_i(a\mid s)-r_i(a'\mid s)
    \right].
    \label{eq:app-pairwise-odds}
\end{equation}
\end{proposition}

\begin{proof}
Take the log-ratio of Eq.~\ref{eq:app-multi-anchor-target} for $a$ and
$a'$. The shared normalizer $Z_{\mathbf{w}}(s)$ cancels.
\end{proof}

Equation~\ref{eq:app-pairwise-odds} gives a precise interpretation of local
agreement and conflict. If all positively weighted anchors increase the score
of $a$ relative to $a'$, then the joint target increases the corresponding
odds relative to the behavior policy. If the anchors disagree, the signed
weighted sum---rather than any one anchor or an ordering of anchors---determines
the change in odds.

\begin{corollary}[Permutation, baseline, and scale properties]
\label{cor:permutation-scale}
The joint target is invariant to any simultaneous permutation of the anchor
scores and their weights. It is also invariant to replacing each score with
$r_i(a\mid s)+c_i(s)$ for any action-independent state baseline $c_i(s)$.
Moreover, it depends on $\mathbf{w}$ and $\alpha$ only through the ratios
$w_i/\alpha$. In particular, for every $c>0$,
\begin{equation}
    q_{c\mathbf{w}}^{(\alpha)}
    =
    q_{\mathbf{w}}^{(\alpha/c)},
    \label{eq:app-weight-temperature-equivalence}
\end{equation}
where the superscript denotes the KL coefficient used to construct the
target.
\end{corollary}

\begin{proof}
Permutation invariance follows from commutativity of addition. Adding the
state baselines multiplies every unnormalized target probability by the same
factor
$\exp(\sum_i w_i c_i(s)/\alpha)$, which cancels in the normalizer.
Finally, Eq.~\ref{eq:app-multi-anchor-target} depends on the weights and KL
coefficient only through $\sum_i(w_i/\alpha)r_i$.
\end{proof}

\begin{remark}[Joint target versus sequential training]
Corollary~\ref{cor:permutation-scale} concerns the explicitly constructed
distribution $q_{\mathbf{w}}$ at a fixed state. It does not imply that
sequentially training a parameterized model on separate anchor targets is
order invariant: different stages can change the model, the visited states,
and the approximation error before the next stage begins.
\end{remark}

\begin{remark}[Fixed point of the finite-candidate surrogate]
The implementation uses the squared log-residual loss in
Eq.~\ref{eq:finite-candidate-loss}, rather than evaluating the
full-distribution KL directly. At a fixed cached state, suppose
$\widehat b$, $\widehat q$, and $\widehat\pi_\theta$ are strictly positive
on the candidate-plus-residual outcome space, $m_b(s)>0$ is fixed, and the
numerical clamps are inactive. This loss is a positively weighted sum of
squared log residuals, so it is non-negative and equals zero if and only if
$\widehat\pi_\theta=\widehat q$. If $\widehat\pi_\theta$ is differentiable,
the loss gradient with respect to $\theta$ also vanishes at that target,
since every residual is zero. This applies to both single-anchor and
composed finite targets.
Thus, the surrogate preserves the zero-loss, zero-gradient property at its
finite target, but does not inherit the full-distribution KL objective
equivalence established above.
\end{remark}

\paragraph{Scope of the guarantees.}
The propositions establish exact target construction and objective
equivalence for finite aligned scores, a full action distribution, and the
fixed cached state distribution. They do not guarantee that a finite
parameterization represents every target exactly, that an optimizer reaches
the best point in that parameterization, or that cached states cover states
visited after a large policy change. They also do not by themselves imply an
improvement in task accuracy or response length. Those are empirical
properties of the selected anchor shifts. Finally, finite-support scoring and
cross-tokenizer alignment introduce practical approximations to the
full-distribution setup analyzed here. Appendix~\ref{app:cross-tokenizer}
specifies those choices.

\section{Supplementary Experimental Details}
\label{app:experimental-details}

This section specifies the models, data, optimization, score alignment, and
evaluation protocols used in our experiments.  The method in
Section~\ref{sec:method} is stated over a common action space and full
distributions. Below we distinguish this idealized definition from the
finite-candidate implementation used to cache policy shifts.

\subsection{Models, Anchor Pairs, and Training Data}
\label{app:models-data}

\paragraph{Models and anchors.}
We conduct experiments with Qwen3-1.7B, Qwen3-4B,
Qwen3-4B-Thinking-2507, Qwen3.5-4B, and OLMo-3-7B-Think as student
models.  We use three anchor pairs:
Qwen3-8B-Base $\rightarrow$ Klear-Reasoner-8B (Klear),
DeepSeek-R1-Distill-Qwen-1.5B $\rightarrow$ DECS-1.5B (DECS), and
MiMo-7B-Base $\rightarrow$ MiMo-7B-RL-0530 (MiMo-RL).  Klear and MiMo-RL
provide accuracy-oriented shifts, while DECS provides an
efficiency-oriented shift.

\paragraph{Mathematics training data.}
We select 3,200 prompts from the mathematics split of
Skywork-OR1-RL-Data~\citep{he2025skywork} and convert them to the DAPO prompt
format~\citep{yu2026dapo}.  Prompts exceeding the student context limit are
discarded.

\paragraph{Code training data.}
We select 3,200 prompts from
KlearReasoner-CodeSub-15K~\citep{su2025klear} after length filtering and
prompt deduplication.  We additionally screen the selected data against
LiveCodeBench v6 and find no overlapping problems.

\paragraph{Behavior trajectories.}
For each student and domain, we sample four responses per prompt from the
initial student $\pi_b$, using temperature 1, top-$p=1$, and a maximum of
2,048 response tokens.  This produces 12,800 trajectories.  All anchors score
the same cached student trajectories, which are subsequently reused
throughout training.

\subsection{Training Hyperparameters}
\label{app:training-hyperparameters}

Single- and multi-anchor runs share the settings in
Table~\ref{tab:appendix-training-hyperparameters}. Each run replays 12,800
cached trajectories for two epochs: 100 rounds of 256 trajectories, with four
optimizer updates per round at a global batch size of 64 trajectories
(400 updates in total). Mathematics and code checkpoints are trained
separately on their corresponding datasets.

\begin{table}[!htbp]
\centering
\caption{Shared training hyperparameters for single- and multi-anchor configurations.}
\label{tab:appendix-training-hyperparameters}
\begin{tabular}{@{}lr@{\hspace{1.5em}}lr@{}}
\toprule
\textbf{Hyperparameter} & \textbf{Value} & \textbf{Hyperparameter} & \textbf{Value} \\
\midrule
Replay rounds & 100 & Global batch size & 64 \\
Cached trajectories per round & 256 & Learning rate & $1.0{\times}10^{-6}$ \\
Training epochs & 2 & LR schedule & Constant \\
Maximum response length & 2,048 & Optimizer & Adam \\
Rollout temperature & 1.0 & Adam $(\beta_1,\beta_2)$ & $(0.9,0.999)$ \\
Rollout top-$p$ & 1.0 & Weight decay & 0.01 \\
$\alpha$ & 2.0 & Gradient clipping & 1.0 \\
\bottomrule
\end{tabular}
\end{table}

We use $\alpha=2.0$ for all single- and multi-anchor runs. Multi-anchor weights
are normalized to sum to one, with $w_i=1/K$ for equal composition of $K$
anchors. Consequently, a composition endpoint with one weight equal to one
uses the same target and training configuration as the corresponding
single-anchor run.

\subsection{Training Cost}
\label{app:training-cost}

We compare online and offline costs in four mathematics settings: Qwen3-1.7B
with two anchors and Qwen3-4B with one, two, or three anchors.
Runs use 100 rounds of 256
trajectories (64 prompts $\times$ four responses), capped at 2,048 tokens.
Offline training replays 12,800 unique cached trajectories for two epochs
(25,600 uses). All GPU stages use eight H100 GPUs.

\begin{table}[!b]
\centering
\caption{GPU-hour costs for one composition in each of four 100-round settings.
CPU target construction is excluded.
Totals use unrounded stage costs.}
\label{tab:training-cost}
\small
\setlength{\tabcolsep}{6pt}
\begin{tabular}{@{}lrrrr@{}}
\toprule
Student & Qwen3-1.7B & \multicolumn{3}{c}{Qwen3-4B} \\
\cmidrule(l){2-2}\cmidrule(l){3-5}
Number of anchors & 2 & 1 & 2 & 3 \\
\midrule
Online DOPD & 15.65 & 15.49 & 17.64 & 19.66 \\
\method & \textbf{7.17} & \textbf{8.31} & \textbf{9.58} & \textbf{11.55} \\
Online/offline cost ratio & $2.18\times$ & $1.86\times$ & $1.84\times$ & $1.70\times$ \\
\midrule
\multicolumn{5}{@{}l}{\textit{\method breakdown}} \\
\quad Student rollout collection & 1.59 & 3.50 & 3.50 & 3.50 \\
\quad Student-reference scoring & 0.49 & 0.62 & 0.62 & 0.62 \\
\quad Anchor scoring & 2.93 & 1.76 & 2.93 & 4.95 \\
\quad Student-only training & 2.17 & 2.43 & 2.53 & 2.48 \\
\bottomrule
\end{tabular}
\end{table}

For the four Qwen3 settings, \method reduces measured GPU-hour cost per
composition by $1.70$--$2.18\times$ (Table~\ref{tab:training-cost}). Both
pipelines score each anchor pair, so this reduction need not grow with anchor
count. More importantly, \method avoids concurrent
anchor serving during student training: Online DOPD keeps all $2K$ anchor
policies in the training loop, whereas \method scores pairs sequentially
before student-only training.

Reusing cached trajectories, student-reference scores, and anchor shifts
across weight configurations makes additional compositions student-only GPU
training runs. Using the Qwen3-4B three-anchor costs, we project five compositions at
$9.07 + 5\times2.48 = 21.47$ GPU hours with \method, compared with
$5\times19.66 = 98.30$ GPU hours online: a $4.58\times$ amortized cost
reduction, assuming unchanged training cost per composition.

\subsection{Cross-Tokenizer Policy-Shift Alignment}
\label{app:cross-tokenizer}

Some students and anchors have different tokenizers, so a student action need
not be an atomic action under an anchor.  We align decisions through the
serialized byte-level token strings, which define a tokenizer-independent
view of the text prefix.  For a generated student prefix $s_t$, we first
losslessly project the entire observed prefix into each anchor tokenizer.
Non-atomic student tokens may decompose into several anchor tokens for this
context construction.  Added tokens are accepted only when decoding and
re-encoding preserves their text exactly.

For a candidate student action $a$, let $\phi_i(a)$ denote its representation
under anchor $i$.  We retain the position for anchor $i$ only if every cached
student candidate maps atomically, i.e., $\phi_i(a)$ is exactly one anchor
token for all 16 candidates.  Its score is then evaluated at the anchor
boundary immediately preceding that token:
\begin{equation}
    \widehat r_i(a\mid s_t)
    =
    \log \pi_i^+\!\left(\phi_i(a)\mid\phi_i(s_t)\right)
    -
    \log \pi_i^-\!\left(\phi_i(a)\mid\phi_i(s_t)\right).
    \label{eq:cross-tokenizer-shift}
\end{equation}
This rule deliberately masks a boundary at which a candidate expands into
multiple anchor tokens. It does not substitute the likelihood of the first
piece, whose meaning would depend on future pieces.  It also ensures that the
pre- and post-anchor likelihoods are conditioned on exactly the same text
prefix.  For a multi-anchor target, we take the logical intersection of the
anchor masks and sum the valid shifts at the corresponding student decision,
as in Eq.~\ref{eq:composed-shift}.  A masked position is excluded from both
the loss and its token normalization. There is no fallback score.  Across the
Klear--DECS experiments in the main table, 93.8--100\% of generated token
positions survive this test, and every cached trajectory contains at least
one valid position.

\paragraph{Finite-candidate target.}
The full-vocabulary target in Eq.~\ref{eq:maps-target} is the conceptual
objective.  In the implementation, we preserve the behavior probability of
the 16 cached student candidates and aggregate all remaining vocabulary mass
into one \emph{other} outcome.  Candidate probabilities are gathered from
full-vocabulary-normalized logits. They are not renormalized within the top
16.  For the single-anchor case, the cached target over
$\mathcal{C}(s)\cup\{\mathrm{other}\}$ is
\begin{align}
    \widehat q(a\mid s)
    &\propto
    \pi_b(a\mid s)\exp\!\left(\widehat r(a\mid s)/\alpha\right),
    &&a\in\mathcal{C}(s), \nonumber\\
    \widehat q(\mathrm{other}\mid s)
    &\propto
    1-\sum_{a\in\mathcal{C}(s)}\pi_b(a\mid s),
    \label{eq:finite-candidate-target}
\end{align}
where the unobserved residual is assigned zero shift.  For multiple anchors,
$\widehat r$ is replaced by $\sum_iw_i\widehat r_i$ after intersecting the
validity masks.  Retaining the residual outcome avoids treating the cached
top-16 set as the complete vocabulary and makes the approximation exact when
all non-candidate actions have zero composed shift.

We fit this finite target with a squared log-residual surrogate.  Let
$\widehat b$ and $\widehat\pi_\theta$ denote the behavior and current-student
distributions over the 17 outcomes, and let
$m_b(s)=\sum_{a\in\mathcal{C}(s)}\pi_b(a\mid s)$ be the behavior probability
mass covered by the cached candidates.  The implemented per-state loss is
\begin{equation}
    \widehat{\mathcal{L}}_{\mathrm{TT}}(s)
    =
    \frac{\alpha}{2m_b(s)}
    \sum_{z\in\mathcal{C}(s)\cup\{\mathrm{other}\}}
    \widehat b(z\mid s)
    \left[
        \log\widehat\pi_\theta(z\mid s)
        -
        \log\widehat q(z\mid s)
    \right]^2.
    \label{eq:finite-candidate-loss}
\end{equation}
We clamp residual masses below $10^{-8}$ and log arguments below $10^{-30}$
for numerical stability.  Equation~\ref{eq:finite-candidate-loss} is a
finite-support surrogate for, rather than a numerical evaluation of, the
full-vocabulary KL in Eq.~\ref{eq:maps-loss}.  It preserves the desired
fixed point because its value and gradient vanish at
$\widehat\pi_\theta=\widehat q$.  The division by $m_b(s)$ additionally makes
its initialization gradient match the candidate-renormalized Direct-OPD
surrogate used to define the cached update.

\subsection{Benchmarks and Test Sets}
\label{app:benchmarks}

\paragraph{Mathematics.}
We evaluate on AIME 2024, AIME 2025, and HMMT February 2025.  Each benchmark
contains 30 competition-level mathematics problems covering topics such as
algebra, geometry, combinatorics, and number theory.

\paragraph{Code.}
We evaluate code generation on LiveCodeBench~\citep{jain2025livecodebench},
which contains recently released competitive-programming problems and reduces
the risk of test-set contamination.  Following the official release filters,
LCB v5 and v6 contain 279 and 131 problems, respectively.  All methods are
evaluated on the same problem sets.

\subsection{Evaluation Prompts and Answer Extraction}
\label{app:evaluation-prompts}

\paragraph{Mathematics prompt.}
Before application of the model's native chat template, each mathematics
example is represented exactly as
\begin{quote}
\small\ttfamily\raggedright
Question: \{Problem\}\\
Answer:
\end{quote}
We enable the model's native reasoning mode where available and stop on its
end-of-sequence or end-of-turn token.  For every problem we independently
sample 64 responses at temperature 0.6 and top-$p=0.95$, allowing at most
32,768 generated tokens.  We extract the last balanced
\texttt{\textbackslash boxed\{...\}} expression, remove superficial
formatting such as whitespace, delimiters, and units, and canonicalize integer
answers before exact comparison with the reference.  We do not use a
computer-algebra system to grant additional symbolic equivalences.  A
length-truncated response is still scored if it contains a valid extracted
answer.

\paragraph{LiveCodeBench prompt.}
We use LiveCodeBench's \texttt{CodeQwenInstruct} generation format.  Its user
message begins with:
\begin{quote}
\small\ttfamily\raggedright
You will be given a question (problem specification) and will generate a
correct Python program that matches the specification and passes all tests.
You will NOT return anything except for the program.\\[2pt]
Question: \{question\_content\}
\end{quote}
For problems with starter code, the message then asks the model to complete
the supplied code and enclose the answer in code delimiters.  Otherwise it
instructs the model to read from standard input, write to standard output, and
return the program inside a Python code block initialized with
\texttt{\# YOUR CODE HERE}.  The official system wrapper is
\texttt{You are a helpful assistant.}, expressed with the Qwen instruction
delimiters.  We sample four responses per problem at temperature 0.6 and
top-$p=0.95$, with a maximum of 40,960 generated tokens.

We use LiveCodeBench's official code extraction and test execution.  A sample
is correct only if its extracted program passes every test for that problem.
Tests run in isolated, read-only containers with the benchmark's fast-test
mode and a six-second timeout per test.  Runtime errors, timeouts, extraction
failures, and incomplete programs are scored as incorrect.

\subsection{Metrics and Aggregation}
\label{app:metrics}

Let $c_{j,r}\in\{0,1\}$ indicate whether response $r$ to benchmark problem
$j$ is correct.  For a benchmark with $N$ problems and $R$ samples per
problem, we report
\begin{equation}
    \operatorname{Acc}
    =
    \frac{100}{NR}\sum_{j=1}^{N}\sum_{r=1}^{R}c_{j,r}.
    \label{eq:evaluation-accuracy}
\end{equation}
Here $R=64$ for mathematics and $R=4$ for LiveCodeBench.  Code accuracy is
therefore an average over four generations (Avg@4), not pass@4: a problem with
one successful sample contributes $1/4$, rather than being counted as fully
solved.  Mean response length is computed over the same $NR$ samples using
the evaluated student's tokenizer and counts generated tokens only.

The Average accuracy and response-token columns are unweighted arithmetic
means of the five benchmark-level values. Following
O1-Pruner~\citep{luo-etal-2026-o1}, AES compares a trained policy with its
corresponding base student. We compute the score separately for each
benchmark before averaging. For benchmark $k$, let
\begin{equation}
    \Delta_{L,k}=\frac{L_{\mathrm{base},k}-L_k}{L_{\mathrm{base},k}},
    \qquad
    \Delta_{A,k}=\frac{A_k-A_{\mathrm{base},k}}{A_{\mathrm{base},k}},
\end{equation}
where $A_k$ and $L_k$ are the policy's accuracy and mean response length,
and $A_{\mathrm{base},k}$ and $L_{\mathrm{base},k}$ are those of its
corresponding base student under the same evaluation protocol.
The reported AES is the unweighted mean of the five benchmark-level scores:
\begin{equation}
    \operatorname{AES}
    = \frac{1}{5}\sum_{k=1}^{5}
    \left[
    \begin{cases}
        \Delta_{L,k} + 3|\Delta_{A,k}|, & \Delta_{A,k}\geq 0,\\
        \Delta_{L,k} - 5|\Delta_{A,k}|, & \Delta_{A,k}<0.
    \end{cases}
    \right]
    \label{eq:aes}
\end{equation}
Thus, AES rewards both accuracy gains and token savings, while penalizing an
accuracy loss more strongly than an equal relative gain.

\subsection{Implementation Details}
\label{app:reproducibility}

We implement our method in Slime.  Qwen models are trained with Megatron-Core,
while OLMo is trained with PyTorch FSDP.  All of our training runs use
bfloat16 precision on a single eight-GPU node.  We use seed 42 for behavior
generation and seed 1234 for training, and report the final checkpoint of each
run.
Each configuration is trained once. The repeated generations used for
evaluation measure decoding variability rather than training-seed variance.

\section{Additional Experimental Results}
\label{app:additional-results}

We provide additional results on the accuracy--efficiency trade-off and anchor
generalization, including the five-point Qwen3-4B composition sweep and
MiMo-RL compositions on Qwen3-1.7B and Qwen3.5-4B.

\subsection{Klear--DECS Operating Points}
\label{app:complete-tradeoff}

\begin{table}[H]
\centering
\caption{Qwen3-4B Klear--DECS weight sweep across mathematics and code.
We set $w_{\mathrm{DECS}}=\rho_{\mathrm{DECS}}$ and
$w_{\mathrm{Klear}}=1-\rho_{\mathrm{DECS}}$, so the capability weights sum to
one. Mathematics uses Avg@64 and code uses Avg@4.}
\label{tab:appendix-q4-weight-sweep}
\resizebox{\textwidth}{!}{%
\begin{tabular}{rcccccccccc}
\toprule
\multirow{2}[4]{*}{$\boldsymbol{\rho_{\mathrm{DECS}}}$} &
    \multicolumn{2}{c}{\textbf{AIME 2024}} &
    \multicolumn{2}{c}{\textbf{AIME 2025}} &
    \multicolumn{2}{c}{\textbf{HMMT 2025}} &
    \multicolumn{2}{c}{\textbf{LCB v5}} &
    \multicolumn{2}{c}{\textbf{LCB v6}} \\
\cmidrule(r){2-3}\cmidrule(r){4-5}\cmidrule(r){6-7}\cmidrule(r){8-9}\cmidrule(l){10-11}
& \textbf{Acc.}$\uparrow$ & \textbf{\#Tok.}$\downarrow$
& \textbf{Acc.}$\uparrow$ & \textbf{\#Tok.}$\downarrow$
& \textbf{Acc.}$\uparrow$ & \textbf{\#Tok.}$\downarrow$
& \textbf{Acc.}$\uparrow$ & \textbf{\#Tok.}$\downarrow$
& \textbf{Acc.}$\uparrow$ & \textbf{\#Tok.}$\downarrow$ \\
\midrule
0.250 & 78.8 & 14,466 & 74.3 & 17,236 & 48.8 & 17,420 & 59.1 & 13,165 & 52.9 & 13,719 \\
0.375 & 77.4 & 12,998 & 73.2 & 15,317 & 48.7 & 15,372 & 59.6 & 12,406 & 50.8 & 12,652 \\
0.500 & 76.0 & 11,521 & 70.1 & 13,816 & 43.0 & 13,743 & 59.2 & 11,702 & 51.3 & 12,357 \\
0.625 & 74.2 & 10,374 & 66.1 & 12,328 & 38.8 & 12,207 & 57.8 & 11,036 & 49.4 & 11,585 \\
0.750 & 70.0 &  9,566 & 62.3 & 11,131 & 34.7 & 10,595 & 55.6 & 10,705 & 46.9 & 11,284 \\
\bottomrule
\end{tabular}%
}
\end{table}

Table~\ref{tab:appendix-q4-weight-sweep} reports the five standardized
composition fractions for the Qwen3-4B sweep, presenting mathematics and code
in a single view. Increasing $\rho_{\mathrm{DECS}}$ generally shortens
responses, while lower and intermediate fractions preserve higher accuracy.
The consistent movement across benchmarks shows that the composition fraction
provides a controllable accuracy--efficiency axis across student policies.

\subsection{Generalization with MiMo-RL}
\label{app:mimo-generalization}

\begin{table}[!htbp]
\centering
\caption{Additional anchor-composition results on Qwen3-1.7B and
Qwen3.5-4B. We report accuracy (Acc., \%) and mean response tokens (\#Tok.).
Mathematics uses Avg@64, while LCB results use Avg@4.
Within each student block, the best and second-best results are marked in
\textbf{bold} and \underline{underlined}, respectively.}
\label{tab:appendix-mimo-generalization}
\resizebox{\textwidth}{!}{%
\begin{tabular}{lccccccccccccc}
\toprule
\multirow{2}[4]{*}{\textbf{Anchors}} &
    \multicolumn{2}{c}{\textbf{AIME 2024}} &
    \multicolumn{2}{c}{\textbf{AIME 2025}} &
    \multicolumn{2}{c}{\textbf{HMMT 2025}} &
    \multicolumn{2}{c}{\textbf{LCB v5}} &
    \multicolumn{2}{c}{\textbf{LCB v6}} &
    \multicolumn{3}{c}{\textbf{Average}} \\
\cmidrule(r){2-3}\cmidrule(r){4-5}\cmidrule(r){6-7}\cmidrule(r){8-9}\cmidrule(r){10-11}\cmidrule(l){12-14}
& \textbf{Acc.}$\uparrow$ & \textbf{\#Tok.}$\downarrow$
& \textbf{Acc.}$\uparrow$ & \textbf{\#Tok.}$\downarrow$
& \textbf{Acc.}$\uparrow$ & \textbf{\#Tok.}$\downarrow$
& \textbf{Acc.}$\uparrow$ & \textbf{\#Tok.}$\downarrow$
& \textbf{Acc.}$\uparrow$ & \textbf{\#Tok.}$\downarrow$
& \textbf{Acc.}$\uparrow$ & \textbf{\#Tok.}$\downarrow$
& \textbf{AES}$\uparrow$ \\
\midrule
\multicolumn{14}{l}{\textbf{\textit{Qwen3-1.7B}}} \\
Base
    & 45.8 & 17,564 & 40.4 & 17,537 & 21.7 & 17,570
    & 36.9 & 14,663 & 28.2 & 14,868
    & 34.62 & 16,440 & 0.00 \\
Klear-only
    & 53.2 & 18,457 & \underline{41.7} & 19,178
    & \underline{24.1} & 19,751 & 36.6 & 15,206
    & \underline{32.8} & 15,288
    & 37.68 & 17,576 & 0.20 \\
DECS-only
    & 46.2 & \textbf{14,558} & 34.9 & \textbf{14,534}
    & 20.9 & \textbf{14,160} & 31.9 & \textbf{13,518}
    & 31.3 & \textbf{13,458}
    & 33.05 & \textbf{14,046} & $-0.08$ \\
MiMo-RL
    & 51.8 & 19,352 & 38.8 & 20,739 & 23.8 & 19,927
    & 36.6 & 15,849 & 29.8 & 15,882
    & 36.14 & 18,350 & 0.02 \\
Klear+DECS
    & \underline{55.1} & \underline{15,289}
    & 40.8 & \underline{16,151}
    & 22.9 & \underline{15,855}
    & \underline{38.4} & \underline{13,963}
    & \underline{32.8} & 13,870
    & \underline{38.00} & \underline{15,026} & \underline{0.38} \\
Klear+MiMo-RL
    & \textbf{60.2} & 22,004 & \textbf{47.0} & 23,792
    & \textbf{30.0} & 24,171 & \textbf{42.7} & 17,568
    & \textbf{35.1} & 17,798
    & \textbf{42.99} & 21,067 & \textbf{0.44} \\
DECS+MiMo-RL
    & 53.5 & 16,035 & 37.6 & 16,348 & 23.5 & 16,021
    & \underline{38.4} & 14,175 & 30.5 & \underline{13,627}
    & 36.70 & 15,241 & 0.25 \\
\midrule
\multicolumn{14}{l}{\textbf{\textit{Qwen3.5-4B}}} \\
Base
    & 72.9 & 18,316 & 63.5 & 20,322 & 59.2 & 26,481
    & 41.9 & 27,209 & 36.6 & 26,500
    & 54.84 & 23,766 & 0.00 \\
Klear-only
    & 78.2 & 21,522 & 66.7 & 23,346 & 56.7 & 27,096
    & 50.9 & 24,951 & 45.0 & 26,200
    & 59.49 & 24,623 & 0.22 \\
DECS-only
    & 64.2 & 15,052 & 55.8 & 17,335 & 52.1 & 22,228
    & 45.9 & \textbf{22,570} & 38.9 & \textbf{23,495}
    & 51.38 & 20,136 & $-0.16$ \\
MiMo-RL
    & \underline{80.4} & 20,770 & 68.5 & 21,975
    & 58.8 & 25,490 & 40.1 & 26,653 & 37.4 & 27,738
    & 57.04 & 24,525 & 0.09 \\
Klear+DECS
    & \textbf{84.2} & 18,900 & \textbf{74.6} & 20,757
    & \textbf{64.0} & 23,635 & \textbf{57.0} & 24,404
    & \textbf{52.7} & 25,665
    & \textbf{66.47} & 22,672 & \textbf{0.68} \\
Klear+MiMo-RL
    & 78.8 & \textbf{12,789} & \underline{68.9} & \textbf{14,496}
    & \underline{63.6} & \textbf{15,874}
    & \underline{54.1} & 25,628 & 46.6 & 26,243
    & \underline{62.40} & \textbf{19,006} & \underline{0.61} \\
DECS+MiMo-RL
    & 78.3 & \underline{13,314} & 68.5 & \underline{14,931}
    & 62.5 & \underline{18,248}
    & 53.4 & \underline{24,374}
    & \underline{47.3} & \underline{24,936}
    & 62.01 & \underline{19,161} & 0.59 \\
\bottomrule
\end{tabular}%
}
\end{table}

Table~\ref{tab:appendix-mimo-generalization} extends the MiMo-RL compositions
in Table~\ref{tab:anchor-generalization} to two additional students. On
Qwen3-1.7B, Klear+MiMo-RL achieves the highest average accuracy and AES, while
DECS+MiMo-RL improves average accuracy over the base model by 2.08 points with
7.3\% fewer tokens. On Qwen3.5-4B, Klear+MiMo-RL and DECS+MiMo-RL improve
average accuracy by 7.56 and 7.17 points while reducing response tokens by
20.0\% and 19.4\%, respectively. These results demonstrate that capability
composition with MiMo-RL transfers across student sizes and anchor pairings.

\makeatletter
\newenvironment{casefigure}{%
  \par\addvspace{\intextsep}%
  \def\@captype{figure}%
  \centering
}{%
  \par\addvspace{\intextsep}%
}
\makeatother

\section{Case Studies}
\label{app:case-studies}

Aggregate metrics show the overall accuracy--efficiency trade-off, but do not
reveal how different anchor shifts change an individual response. Following
the qualitative-example format used in prior reasoning-efficiency
work~\citep{jiang2026overthinking,luo-etal-2026-o1}, each case below foregrounds
the question and matched policy responses in a single panel. We select the
mathematics cases from the shared Qwen3-4B AIME 2024 evaluation and the code
case from the corresponding LCB v5 evaluation, using fixed correctness and
length criteria before inspecting the response text. Each panel reproduces
the beginning and decisive ending of the sampled response, starting from
\texttt{<think>}. The model's full response length appears in the header, and
the bracketed omission marker reports the exact full-response length under
the student tokenizer. We normalize whitespace and mathematical typesetting
only. These examples illustrate the aggregate trends rather than establish
causal effects.

\subsection{Combining Complementary Accuracy Signals}
\label{app:case-complementary-accuracy}

\begin{casefigure}
\begingroup
\setlength{\fboxsep}{2pt}
\setlength{\fboxrule}{0.5pt}

\fcolorbox{black!25}{white}{%
\begin{minipage}{0.965\linewidth}
\scriptsize
\raggedright
\colorbox{blue!12}{%
\parbox{\dimexpr\linewidth-2\fboxsep\relax}{%
\textbf{\textcolor{blue!65!black}{CASE A}}
\hspace{0.4em}\textbf{MiMo-RL supplies the correct behavior}
\hfill\textbf{AIME 2024 I-5, Sample 22}}}
\par\vspace{2pt}
\colorbox{black!5}{%
\parbox{\dimexpr\linewidth-2\fboxsep\relax}{%
\textbf{Question.}
Rectangles $ABCD$ and $EFGH$ satisfy that $D,E,C,F$ are collinear and
$A,D,H,G$ are concyclic. Given $BC=16$, $AB=107$, $FG=17$, and $EF=184$,
find $CE$. \hfill\textbf{Reference: 104}}}
\par\vspace{3pt}

\colorbox{red!6}{%
\begin{minipage}{\dimexpr\linewidth-2\fboxsep\relax}
\textbf{Klear-only}\hfill
\textcolor{red!70!black}{\textbf{Incorrect: 107}}
\hspace{0.6em}\textbf{29,849 tokens}
\par
\texttt{<think>} Okay, so I need to solve this geometry problem. Let me read
it again carefully. We have two rectangles, $ABCD$ and $EFGH$ \ldots
\hfill\textcolor{black!55}{\emph{[middle omitted, full response: 29,849 tokens]}}
\par
Then \(CE\approx|199-92.092|=106.908\). This is very close to 107, and the
intended answer is likely \(\boxed{107}\).
\end{minipage}}
\par\vspace{2pt}
\colorbox{green!7}{%
\begin{minipage}{\dimexpr\linewidth-2\fboxsep\relax}
\textbf{MiMo-RL}\hfill
\textcolor{green!40!black}{\textbf{Correct: 104}}
\hspace{0.6em}\textbf{14,535 tokens}
\par
\texttt{<think>} Okay, let me try to figure out this geometry problem. So,
we have two rectangles, $ABCD$ and $EFGH$ \ldots
\hfill\textcolor{black!55}{\emph{[middle omitted, full response: 14,535 tokens]}}
\par
The problem assumes that the points are ordered as $D,E,C,F$. Therefore, we
take the case where $e=3$, which gives \(\boxed{104}\).
\end{minipage}}
\par\vspace{2pt}
\colorbox{blue!8}{%
\begin{minipage}{\dimexpr\linewidth-2\fboxsep\relax}
\textbf{Klear+MiMo-RL}\hfill
\textcolor{green!40!black}{\textbf{Correct: 104}}
\hspace{0.6em}\textbf{15,659 tokens}
\par
\texttt{<think>} Okay, so I need to solve this geometry problem. Let me read
it again carefully. We have two rectangles, $ABCD$ and $EFGH$ \ldots
\hfill\textcolor{black!55}{\emph{[middle omitted, full response: 15,659 tokens]}}
\par
We must have $E=(3,0)$ and $F=(187,0)$, so that $0<3<107<187$. Therefore,
\(CE=|107-3|=104\). The four points satisfy the circle equation.
\textbf{Final Answer:} \(\boxed{104}\).
\end{minipage}}
\end{minipage}}

\par\penalty0\vspace{5pt}

\fcolorbox{black!25}{white}{%
\begin{minipage}{0.965\linewidth}
\scriptsize
\raggedright
\colorbox{blue!12}{%
\parbox{\dimexpr\linewidth-2\fboxsep\relax}{%
\textbf{\textcolor{blue!65!black}{CASE B}}
\hspace{0.4em}\textbf{Klear-only supplies the correct behavior}
\hfill\textbf{AIME 2024 I-13, Sample 10}}}
\par\vspace{2pt}
\colorbox{black!5}{%
\parbox{\dimexpr\linewidth-2\fboxsep\relax}{%
\textbf{Question.}
Let $p$ be the least prime for which $p^2$ divides $n^4+1$ for some positive
integer $n$. Find the least positive $m$ for which $p^2$ divides $m^4+1$.
\hfill\textbf{Reference: 110}}}
\par\vspace{3pt}

\colorbox{green!7}{%
\begin{minipage}{\dimexpr\linewidth-2\fboxsep\relax}
\textbf{Klear-only}\hfill
\textcolor{green!40!black}{\textbf{Correct: 110}}
\hspace{0.6em}\textbf{15,507 tokens}
\par
\texttt{<think>} Okay, so I need to find the least prime number $p$ such that
there's a positive integer $n$ where $n^4+1$ is divisible by $p^2$ \ldots
\hfill\textcolor{black!55}{\emph{[middle omitted, full response: 15,507 tokens]}}
\par
\(m=8+17\cdot6=110\). This is the smallest solution among the lifts of the
four solutions modulo 17. \textbf{Final Answer:} \(\boxed{110}\).
\end{minipage}}
\par\vspace{2pt}
\colorbox{red!6}{%
\begin{minipage}{\dimexpr\linewidth-2\fboxsep\relax}
\textbf{MiMo-RL}\hfill
\textcolor{red!70!black}{\textbf{Incorrect: 134}}
\hspace{0.6em}\textbf{14,821 tokens}
\par
\texttt{<think>} Okay, so I need to find the least prime number $p$ such that
there's a positive integer $n$ where $n^4+1$ is divisible by $p^2$ \ldots
\hfill\textcolor{black!55}{\emph{[middle omitted, full response: 14,821 tokens]}}
\par
\(m=15+17\cdot7=134\). So the least positive integer $m$ such that
\(m^4+1\equiv0\pmod{17^2}\) is \(\boxed{134}\).
\end{minipage}}
\par\vspace{2pt}
\colorbox{blue!8}{%
\begin{minipage}{\dimexpr\linewidth-2\fboxsep\relax}
\textbf{Klear+MiMo-RL}\hfill
\textcolor{green!40!black}{\textbf{Correct: 110}}
\hspace{0.6em}\textbf{18,549 tokens}
\par
\texttt{<think>} Okay, so I need to find the least prime number $p$ such that
there's a positive integer $n$ where $n^4+1$ is divisible by $p^2$ \ldots
\hfill\textcolor{black!55}{\emph{[middle omitted, full response: 18,549 tokens]}}
\par
\(110^4+1\equiv0\pmod{289}\). Since it is the smallest among the four lifts,
it is the least positive integer $m$. \textbf{Final Answer:} \(\boxed{110}\).
\end{minipage}}
\end{minipage}}
\endgroup
\par\nobreak
\caption{Complementary accuracy cases. MiMo-RL alone solves Case A and
Klear-only alone solves Case B. Klear+MiMo-RL solves both.}
\label{fig:case-accuracy-composition}
\end{casefigure}

Figure~\ref{fig:case-accuracy-composition} shows complementary outcomes:
MiMo-RL supplies the correct answer in Case A, whereas Klear-only does so in
Case B. Klear+MiMo-RL solves both, preserving behavior contributed by either
accuracy-oriented anchor.

\FloatBarrier
\subsection{Combining Accuracy and Efficiency}
\label{app:case-accuracy-efficiency}

Figure~\ref{fig:case-accuracy-efficiency} shows the same composition pattern
across mathematics and code: Klear-only is correct but long, DECS-only is
shorter but incorrect, and Klear+DECS is both correct and shortest.

\begin{casefigure}
\begingroup
\setlength{\fboxsep}{2pt}
\setlength{\fboxrule}{0.5pt}

\fcolorbox{black!25}{white}{%
\begin{minipage}{0.965\linewidth}
\fontsize{7.5}{8.5}\selectfont
\raggedright
\colorbox{blue!12}{%
\parbox{\dimexpr\linewidth-2\fboxsep\relax}{%
\textbf{\textcolor{blue!65!black}{CASE C}}
\hspace{0.4em}\textbf{Accuracy--efficiency composition in mathematics}
\hfill\textbf{AIME 2024 I-14, Sample 28}}}
\par\vspace{2pt}
\colorbox{black!5}{%
\parbox{\dimexpr\linewidth-2\fboxsep\relax}{%
\textbf{Question.}
A tetrahedron has opposite edge-length pairs
$(\sqrt{41},\sqrt{41})$, $(\sqrt{80},\sqrt{80})$, and
$(\sqrt{89},\sqrt{89})$. Its inradius is $m\sqrt{n}/p$ in reduced form.
Find $m+n+p$. \hfill\textbf{Reference: 104}}}
\par\vspace{3pt}

\colorbox{green!7}{%
\begin{minipage}{\dimexpr\linewidth-2\fboxsep\relax}
\textbf{Klear-only}\hfill
\textcolor{green!40!black}{\textbf{Correct: 104}}
\hspace{0.6em}\textbf{21,323 tokens}
\par
\texttt{<think>} Okay, so I need to find the distance from a point $I$ inside
the tetrahedron $ABCD$ to each of its faces \ldots
\hfill\textcolor{black!55}{\emph{[middle omitted, full response: 21,323 tokens]}}
\par
\(m=20\), \(n=21\), and \(p=63\). \textbf{Final Answer:}
\(m+n+p=20+21+63=\boxed{104}\).
\end{minipage}}
\par\vspace{2pt}
\colorbox{red!6}{%
\begin{minipage}{\dimexpr\linewidth-2\fboxsep\relax}
\textbf{DECS-only}\hfill
\textcolor{red!70!black}{\textbf{Incorrect: 125}}
\hspace{0.6em}\textbf{10,101 tokens}
\par
\texttt{<think>} Okay, so I need to find the distance from a point $I$ inside
the tetrahedron $ABCD$ to each of its faces \ldots
\hfill\textcolor{black!55}{\emph{[middle omitted, full response: 10,101 tokens]}}
\par
\(m=20\), \(n=42\), and \(p=63\). \textbf{Final Answer:}
\(m+n+p=20+42+63=\boxed{125}\).
\end{minipage}}
\par\vspace{2pt}
\colorbox{blue!8}{%
\begin{minipage}{\dimexpr\linewidth-2\fboxsep\relax}
\textbf{Klear+DECS}\hfill
\textcolor{green!40!black}{\textbf{Correct: 104}}
\hspace{0.6em}\textbf{7,601 tokens}
\par
\texttt{<think>} Okay, so I need to solve this geometry problem about a
tetrahedron. The tetrahedron is $ABCD$ with specific edge lengths \ldots
\hfill\textcolor{black!55}{\emph{[middle omitted, full response: 7,601 tokens]}}
\par
The distance is \(\frac{20\sqrt{21}}{63}\). This is in the required form with
$m=20$, $n=21$, and $p=63$. \textbf{Final Answer:}
\(m+n+p=20+21+63=\boxed{104}\).
\end{minipage}}
\end{minipage}}

\par\penalty0\vspace{5pt}

\fcolorbox{black!25}{white}{%
\begin{minipage}{0.965\linewidth}
\fontsize{7.5}{8.5}\selectfont
\raggedright
\colorbox{blue!12}{%
\parbox{\dimexpr\linewidth-2\fboxsep\relax}{%
\textbf{\textcolor{blue!65!black}{CASE D}}
\hspace{0.4em}\textbf{Accuracy--efficiency composition in code}
\hfill\textbf{LCB v5, Question 3487, Sample 2}}}
\par\vspace{2pt}
\colorbox{black!5}{%
\parbox{\dimexpr\linewidth-2\fboxsep\relax}{%
\textbf{Question: Find Maximum Removals From Source String.}
Given \texttt{source}, a subsequence \texttt{pattern}, and
\texttt{targetIndices}, remove the maximum number of indexed characters while
keeping \texttt{pattern} a subsequence. Here \(n\leq3{,}000\).}}
\par\vspace{3pt}

\colorbox{green!7}{%
\begin{minipage}{\dimexpr\linewidth-2\fboxsep\relax}
\textbf{Klear-only}\hfill
\textcolor{green!40!black}{\textbf{Correct: Pass}}
\hspace{0.6em}\textbf{17,076 tokens}
\par
\texttt{<think>} Okay, let's see. The problem is about finding the maximum
number of operations we can perform where each operation removes a character
from \texttt{source} \ldots
\hfill\textcolor{black!55}{\emph{[middle omitted, full response: 17,076 tokens]}}
\par
We use a dynamic programming approach to find the minimum cost of indices in
\texttt{targetIndices} that are used in forming the pattern:
{\ttfamily dp = [[INF] * (m + 1) for \_ in range(n + 1)]\\
return len(targetIndices) - dp[n][m].}
\end{minipage}}
\par\vspace{2pt}
\colorbox{red!6}{%
\begin{minipage}{\dimexpr\linewidth-2\fboxsep\relax}
\textbf{DECS-only}\hfill
\textcolor{red!70!black}{\textbf{Incorrect: Wrong Answer}}
\hspace{0.6em}\textbf{15,726 tokens}
\par
\texttt{<think>} Okay, I need to solve this problem where I have to find the
maximum number of operations I can perform by removing characters from the
source string \ldots
\hfill\textcolor{black!55}{\emph{[middle omitted, full response: 15,726 tokens]}}
\par
But the example says 1. This suggests that the code is incorrect. \ldots\
Given the time constraints, I will proceed with the code that checks for the
maximum \(k\) by removing the last \(k\) indices:
{\ttfamily removed = set(targetIndices[-mid:]).}
\end{minipage}}
\par\vspace{2pt}
\colorbox{blue!8}{%
\begin{minipage}{\dimexpr\linewidth-2\fboxsep\relax}
\textbf{Klear+DECS}\hfill
\textcolor{green!40!black}{\textbf{Correct: Pass}}
\hspace{0.6em}\textbf{9,389 tokens}
\par
\texttt{<think>} Okay, let's see. The problem is about removing characters
from the source string, but only from indices in \texttt{targetIndices}
\ldots
\hfill\textcolor{black!55}{\emph{[middle omitted, full response: 9,389 tokens]}}
\par
The DP array \texttt{dp} is used where \texttt{dp[i]} represents the minimal
cost to form the first \(i\) characters of the pattern:
{\ttfamily dp = [INF] * (m + 1)\\
dp[0] = 0\\
for j in range(m-1, -1, -1): \ldots\\
return len(targetIndices) - dp[m].}
\end{minipage}}
\end{minipage}}
\endgroup
\par\nobreak
\caption{Accuracy--efficiency composition in mathematics (Case C) and code
(Case D). Klear+DECS retains Klear-only's correct outcome while producing the
shortest response in both cases.}
\label{fig:case-accuracy-efficiency}
\end{casefigure}

\FloatBarrier

\section{Limitations and Future Work}
\label{app:limitations}

\small
\noindent\textbf{Limitations.}
First, due to computational constraints, our experiments focus on models
below 10B parameters, and we have not yet established whether the observed
benefits persist at larger scales. Second, our method assumes access to both
a post-trained model exhibiting the desired capability and its corresponding
pre-anchor, so that the capability-specific policy shift can be recovered.
Such matched anchor pairs may not always be available. Third, although \method improves the accuracy--efficiency balance, its
efficiency gains remain smaller than those achieved by methods optimized
exclusively for efficient reasoning.

\noindent\textbf{Future work.}
We plan to evaluate \method on larger models and a broader range of model
architectures. We will also extend capability composition to new domains,
particularly agentic reasoning and tool-use tasks.
\normalsize

\end{document}